\documentclass[11pt,letterpaper]{article}
\usepackage{arxiv}
\usepackage{amsmath,amssymb,amsthm}
\usepackage{array}
\usepackage{booktabs}
\usepackage{colortbl}
\usepackage{enumitem}
\usepackage{float}
\usepackage{graphicx}
\usepackage{hyperref}
\usepackage{longtable}
\usepackage{microtype}
\usepackage{multirow}
\usepackage{tikz}
\usepackage{xcolor}
\usepackage[most]{tcolorbox}
\usetikzlibrary{arrows.meta,positioning,fit,calc}
\newcommand{\skill}[1]{\texttt{/#1}}

\newcommand{\R}{\mathbb R}

\newcommand{\tableidea}[1]{\textcolor{MLTteal!60!black}{\sffamily\bfseries Idea~#1}}
\newcommand{\githubentry}[2]{
\noindent{\fontsize{8}{10}\selectfont
\scalebox{0.93}{\textcolor{MLTnavy}{\sffamily\bfseries #1:}\hspace{0.45em}%
\href{#2}{\texttt{\detokenize{#2}}}}}\par
}

\newtcolorbox{openquestionbox}{
  enhanced,
  breakable,
  colback=orange!10,
  colframe=orange!10,
  boxrule=0pt,
  arc=1mm,
  left=2mm,
  right=2mm,
  top=1.5mm,
  bottom=1.5mm,
  before skip=6pt,
  after skip=6pt
}
\newenvironment{openquestion}
  {\begin{openquestionbox}\noindent\textbf{Open question.}\ }
  {\end{openquestionbox}}

\tcolorboxenvironment{theorem}{
  enhanced,
  breakable,
  colback=MLTblue!4,
  colframe=MLTblue!4,
  boxrule=0pt,
  borderline west={1.5pt}{0pt}{MLTblue!72},
  arc=1mm,
  left=2.5mm,
  right=2.5mm,
  top=1.5mm,
  bottom=1.5mm,
  before skip=7pt,
  after skip=7pt
}
\tcolorboxenvironment{corollary}{
  enhanced,
  breakable,
  colback=MLTteal!4,
  colframe=MLTteal!4,
  boxrule=0pt,
  borderline west={1.5pt}{0pt}{MLTteal!68},
  arc=1mm,
  left=2.5mm,
  right=2.5mm,
  top=1.5mm,
  bottom=1.5mm,
  before skip=7pt,
  after skip=7pt
}

\tcbset{
  valg-note/.style={
    enhanced,
    breakable,
    colback=MLTblue!3,
    colframe=MLTblue!3,
    boxrule=0pt,
    borderline west={1.4pt}{0pt}{MLTblue!68},
    arc=1mm,
    outer arc=1mm,
    boxsep=0pt,
    left=3mm,
    right=3mm,
    top=2mm,
    bottom=2mm,
    before skip=6pt,
    after skip=6pt
  },
  valg-comparison/.style={
    valg-note,
    colback=MLTblue!4,
    coltitle=MLTnavy,
    fonttitle=\sffamily\bfseries\normalsize,
    attach title to upper={\par\smallskip}
  },
  valg-rubric/.style={
    valg-note,
    colback=MLTteal!3,
    borderline west={1.4pt}{0pt}{MLTteal!72},
    coltitle=MLTteal!55!black,
    fonttitle=\sffamily\bfseries\normalsize,
    attach title to upper={\par\smallskip}
  }
}

\title{\huge\VALG{}\textcolor{MLTnavy}{: An Agentic System for ML Theory Research}\\[0.35em]
{\normalfont\Large\color{MLTink}and Demonstrations on COLT 2026 Open Problems}}

\author{%
\begin{minipage}{0.94\textwidth}
\centering
{\sffamily\normalsize\color{MLTink}
Dechen Zhang\textsuperscript{1,2} \quad
Xuan Tang\textsuperscript{1} \quad
Xinxiang Yin\textsuperscript{1}\\[-0.05em]
Xingwu Chen\textsuperscript{1} \quad
Jian Qian\textsuperscript{1} \quad
Difan Zou\textsuperscript{1,2}}
\\[0.38em]
{\sffamily\footnotesize\color{MLTink!72}
\textsuperscript{1}The University of Hong Kong \quad
\textsuperscript{2}Shenzhen Loop Area Institute }
\\[0.2em]
{\sffamily\footnotesize\color{MLTink!58}
Correspondence: \texttt{dzou@hku.hk}}
\end{minipage}
}
\date{\sffamily\footnotesize\color{MLTink!58}\today}

\begin{document}

%%%%%%%%%%%%%%%%%%%% title/author/date/abstract/content %%%%%%%%%%%%%%%%%%%%

\maketitle

\begin{abstract}
Machine learning theory studies learning procedures through mathematical setups in which the data model, training protocol, oracle access, loss, metric, and randomness define the phenomenon that a theorem is meant to explain. Solving an open problem therefore requires the problem formulation, theorem target, and proof mechanism to be developed in concert. Researchers formulate hypotheses, test them through preliminary theoretical or empirical analysis, and refine both assumptions and proofs. We investigate whether this process can be organized as an autonomous agentic workflow for ML theory research.

We develop \VALG, an agentic system that combines multi-level Verification, Adaptive formulation of Learning-theory problems, and Graph-structured proof development. Within each source-relative theorem branch, \VALG maintains a fixed mathematical specification, checks the theorem-level composition of a typed proof-dependency graph, and constructs and reviews local proofs in dependency order. When a proof attempt fails, \VALG identifies whether the obstruction lies in a derivation, the proof structure, or the theorem formulation and routes the next attempt accordingly. Formulation-level obstructions initiate an explicitly related variant or relaxation, preserving the mathematical relation between the resulting theorem and the source problem.

We evaluate \VALG on nine subproblems from five COLT 2026 open problems spanning tensor decomposition, learning complexity, one-bit mean estimation, differential privacy, and online optimization. Two runs produce internally finalized theorem candidates that match the scope of their source briefs and address the original subproblems; the remaining seven yield restricted-method results, special cases, or conditional theorems. For each candidate, we report the problem setup, theorem statement, and progress relative to the source problem. These case studies show how \VALG keeps source-scope matches, relaxations, conditional results, and blocked attempts mathematically distinct. \VALG is open source.
\begin{tcolorbox}[valg-note]
    \githubentry{VALG}{https://github.com/DechenZhang/VALG-ML-Theory-Agent/tree/main/skills}
    \githubentry{COLT Problem Solutions}{https://github.com/DechenZhang/VALG-ML-Theory-Agent/tree/main/case-studies}  
\end{tcolorbox}

\end{abstract}
\section{Introduction}
\label{sec:introduction}

Machine learning theory studies learning procedures through mathematical
setups that specify how data are generated, what information an algorithm can
use, which loss or risk is measured, and how performance scales with problem
parameters.  This viewpoint is already present in classical PAC learning and
statistical learning theory, where the hypothesis class, distributional model,
sample size, loss, and success probability are part of the statement being
studied~\citep{valiant1984theory,vapnik1998statistical,shalev2014understanding,wainwright2019highdimensional}.
Modern learning-theory papers often make the same structure more explicit:
a result may depend on the training protocol, oracle model, randomness model,
tail condition, approximation notion, or asymptotic regime.  In this literature,
such choices usually appear in the theorem statement itself, because they
specify the learning phenomenon under study.

Consequently, ML-theoretic research often proceeds as a concurrent process in
which the problem setup, theorem formulation, and proof technique are
discovered and refined together.  This flexibility can yield scientifically
useful outcomes: a well-stated relaxation, a conditional theorem, or a
restricted setting can illuminate why the original problem is hard and which
assumptions are essential.  It also creates a risk.  Once the protocol becomes
adaptive, the representation approximate, the claim distribution-dependent,
or the target property assumed, the resulting theorem may no longer address
the original problem.

This observation points to a need for systematic support: when the setup,
theorem formulation, and proof technique co-evolve, a reasoning system for ML
theory must track that evolution explicitly.  Recent language-model systems
have made substantial progress on extended reasoning.  General-purpose agents
combine generation, tool use, feedback, and refinement across multiple
steps~\citep{yao2023react,shinn2023reflexion,madaan2023selfrefine}.
Scientific-agent systems organize research into literature search, hypothesis
generation, experimentation, critique, and paper
writing~\citep{lu2024aiscientist,schmidgall2025agentlab,gottweis2026coscientist},
while research-agent benchmarks evaluate extended machine-learning tasks
rather than single responses~\citep{huang2023mlagentbench,chan2024mlebench,
wijk2024rebench,starace2025paperbench}.  Mathematical research agents coordinate
conjecture generation, long-horizon proof search, criticism, and revision in
natural language, while AI4Theory systems also connect literature synthesis,
theorem proving, algorithm design, and numerical experiments~\citep{
feng2026aletheia,liu2026danus,zheng2026comath,he2026reasflow}.  In parallel,
formal-mathematics agents use proof assistants, retrieval, and repair loops to
prove or formalize fixed statements, including complete research-paper
developments and statistical learning theory libraries~\citep{
yang2023leandojo,kripner2026openprover,zhang2026leanmarathon,zhang2026ai4slt}.
Together, these advances leave an ML-theory-specific pre-formalization question:
how should an agentic system manage informal theorem development when both the
proof and the problem formulation may require revision?

To this end, we develop \VALG\footnote{The name also honors four pioneers of learning theory: Leslie \textbf{V}aliant, whose PAC framework formalized efficient learnability~\citep{valiant1984theory}; Dana \textbf{A}ngluin, whose \(L^\ast\) algorithm established polynomial-time exact learning of regular languages from membership and equivalence queries~\citep{angluin1987learning}; Nick \textbf{L}ittlestone, whose mistake-bound analysis and Winnow algorithm shaped online learning~\citep{littlestone1988learning}; and E. Mark \textbf{G}old, whose identification-in-the-limit framework provided an early mathematical model of language learning~\citep{gold1967language}.}, an agentic system that combines
multi-level Verification, Adaptive formulation of Learning-theory problems,
and Graph-structured proof development within source-relative theorem branches.
The task requires generating a
plausible proof while tracking the theorem under consideration, every added
assumption, the cause of each failed attempt, and the relation of any revised
formulation to the source problem.  Each branch in \VALG fixes one mathematical
specification and represents its proof as a typed dependency graph from
primitive assumptions, through intermediate lemmas, to the target theorem.
Before local proof work begins, a mechanism-aware global analysis checks the
mathematical source of each hard claim, the compatibility of lemma interfaces,
and the closure of recursive or limiting arguments.  Proof development then
proceeds from global structure to local derivations.  We localize failures to a
derivation, the proof graph, or the theorem formulation; only formulation-level
obstructions create a new variant or relaxation.  Specialized reviewers
separately examine structure, rigor, citation use, and adversarial boundary
cases before assigning a source-relative outcome.

We evaluate \VALG on nine subproblems drawn from five COLT 2026 open-problem papers: tensor ALS/GD overparameterization~\citep{arvanitakis2026als},
distribution-independent deep or statistical-query learning versus linear
dimension complexity~\citep{feldman2026deep}, non-adaptive one-bit mean
estimation~\citep{lau2026interaction}, differential privacy PAC learning~\citep{nissim2026private} and online optimization of Piecewise-Lipschitz functions~\citep{balcan2026online}.  These questions remain unresolved in
the source works and provide a substantive testbed beyond synthetic
proof-generation tasks. Across nine runs, \VALG produces 22 internally finalized theorem candidates. \textbf{Two runs produce finalized theorem candidates that fully match the scope of their source subproblems; the remaining seven yield restricted-method results, special cases, or conditional theorems.} For each run, we report the problem setup and finalized theorem for selected representative candidates, together with an assessment of their progress toward the corresponding open problem.

\paragraph{Contributions.}
Our contributions are highlighted as follows:
\begin{enumerate}[leftmargin=*,itemsep=1pt]
  \item  We formulate informal ML-theory research
  as an agentic theorem-development problem, in which the learning setup,
  assumptions, theorem target, proof dependencies, and relation to the
  originating research question must be developed and tracked together. This
  formulation reflects the fact that ML theory often requires refinement of both
  the mathematical assumptions and the object being proved.
  \item We develop \VALG,
  which organizes this process through source-relative perspective--idea
  branches, fixed theorem contracts, typed proof-dependency graphs, and a
  sketch--global--step--assembly proof pipeline. The architecture separates
  open-ended problem design from contract-based proof development and assigns
  different review mechanisms to the two stages.
  \item We introduce a hierarchical
  diagnosis and revision mechanism that localizes failures to the derivation,
  proof graph, or theorem formulation. Local failures trigger local repairs,
  structural failures trigger proof-architecture revisions, and only
  formulation-level obstructions create an explicitly source-related variant or
  relaxation. This allows \VALG to retain the mathematical relationship between
  a revised theorem and the original ML-theory problem rather than silently
  replacing the target.
  \item We apply \VALG to nine subproblems
  drawn from five COLT 2026 open-problem papers and obtain 22 internally
  finalized theorem candidates. Two runs produce candidates that match the
  scope of their source subproblems, while the remaining seven produce
  restricted-method results, special cases, or conditional theorems. The case
  studies illustrate how an agentic system can distinguish source-scope matches
  from related but weaker results and unsuccessful branches when developing
  theory for unresolved ML problems.
\end{enumerate}

The retrospective focuses on branch records and their mathematical outcomes. 

\paragraph{Organization.}
Section~\ref{sec:related} positions the work relative to informal mathematical-reasoning agent, formal theorem proving and research-agent benchmarks.
Section~\ref{sec:workflow} presents the theorem-development method,
and section~\ref{sec:evaluation} then examines the COLT open problems and the resulting theorem candidates.  The paper closes with limitations and open directions.

\section{Related Work}
\label{sec:related}

\subsection{Mathematical discovery and research agents}

AI for mathematics includes neural--symbolic systems for conjecture formation,
geometry, and executable program search~\citep{davies2021guiding,
trinh2024alphageometry,romeraparedes2024funsearch,novikov2025alphaevolve}, as
well as more open-ended systems that combine conjecturing, computation, and
symbolic solving~\citep{chen2026moonshine,chen2026iteris,xia2026neurosymbolic}.
General scientific research agents extend this pattern across ideas,
experiments, literature, and writing~\citep{lu2024aiscientist,
schmidgall2025agentlab,gottweis2026coscientist,yang2026aris}; mathematical research places
additional weight on constructing and checking long derivations.

Research-level proof agents organize natural-language mathematics around
generation, criticism, and revision.  Aletheia scales a
generator--verifier--reviser loop across pure mathematics, while ProofCouncil
and RMA distribute analysis, literature use, proof construction, criticism, and
computation across specialized roles~\citep{feng2026aletheia,
schmitt2026proofcouncil,zhao2026rma}.  Rethlas, QED, and the MechMath Agent Team
connect informal exploration with formal checking through theorem retrieval,
decomposition, and natural- or formal-language provers~\citep{ju2026rethlas,
an2026qed,cao2026mechmath}; the AI Co-Mathematician instead supports interactive
refinement of definitions, questions, conjectures, computations, and proof
directions~\citep{zheng2026comath}.

Danus addresses long-horizon coordination through parallel proof search and a
shared fact graph~\citep{liu2026danus}.  A main agent redirects workers across
lines of attack, while a stateless verifier checks each natural-language claim
before the claim, its proof, and its logical dependencies enter the graph.  This
structure lets many local contributions accumulate into a long argument and
supports transitive removal when an accepted fact is later rejected.  Its fact
graph is close to our typed proof dependencies, although Danus terminates when a
supplied target or its refutation becomes a verified fact.  \VALG additionally
separates a proposed proof architecture from theorem-level feasibility and local
derivations, and it treats a formulation-level obstruction as a reason to create
an explicitly source-related theorem branch rather than to edit the target in
place.

The closest AI4Theory systems extend agentic research beyond pure proof:
ReasFlow integrates literature synthesis, algorithm design, theorem proving,
numerical experiments, and manuscript preparation, while Iteris combines proof
drafts with numerical and adversarial exploration~\citep{he2026reasflow,
chen2026iteris}.  \VALG focuses more specifically on ML-theory questions whose
learning setup, assumptions, and theorem target may co-evolve.  It records
whether a result addresses the source question itself, a restricted method, a
special case, or a conditional variant; this source-relative distinction guides
both failure routing and case-study evaluation.

\subsection{Formal theorem proving agents}

Formal theorem proving begins after a mathematical statement and its assumptions
have been encoded in a proof assistant.  The kernel then supplies an exact
acceptance criterion for a completed proof.  LeanDojo provides a programmatic
Lean environment, premise annotations, and retrieval-augmented proving
benchmarks, while HyperTree Proof Search combines learned proof-step proposals
with structured search~\citep{yang2023leandojo,lample2022hypertree}.  COPRA wraps
a general-purpose language model in stateful backtracking search and uses proof
assistant errors and retrieved lemmas as feedback~\citep{thakur2024copra}.

Subsequent systems improve formal proof search through synthetic data,
proof-assistant feedback, informal planning, and reusable lemma libraries.
DeepSeek-Prover and its later versions combine large-scale Lean data,
reinforcement learning, tree search, and subgoal decomposition; Lean-STaR
interleaves informal thoughts with tactic generation; and LEGO-Prover and
DreamProver grow transferable libraries of verified lemmas across
problems~\citep{xin2024deepseekprover,ren2025deepseekproverv2,
lin2025leanstar,wang2023legoprover,zhang2026dreamprover}.  OpenProver, OProver,
Numina-Lean-Agent, and LAMP expose agentic interaction, verifier-guided repair,
and tool use in Lean, while Discover and Prove separates answer discovery from
formal proof for hard-mode statements~\citep{kripner2026openprover,
ma2026oprover,liu2026numinaleanagent,srinivasan2026lamp,
liu2026discoverprove}.  These methods can provide kernel-checked correctness for
the encoded theorem, but they generally take the formal target as fixed; they do
not decide whether a revised learning-theory formulation remains a meaningful
answer to an originating research question.

AI4SLT brings this formal perspective directly to machine learning theory.  It
combines human-designed proof strategies with AI-assisted Lean construction to
develop an empirical-process library containing Gaussian concentration,
Dudley's entropy integral, and sharp least-squares regression rates~\citep{
zhang2026ai4slt}.  The formalization reveals implicit assumptions and omitted
steps in standard presentations, demonstrating the value of machine-checked ML
theory.  Its objective is to formalize an established mathematical development;
\VALG instead operates earlier, when the learning setup and theorem target are
still under construction, and can pass a finalized branch to such a formal
verification workflow.

\subsection{Autoformalization}

Autoformalization addresses a different interface: translating natural-language
mathematics into a formal statement or proof.  Early LLM-based work demonstrated
statement translation into Isabelle, ProofNet paired informal undergraduate
problems with Lean statements and proofs, and MMA expanded training through
large-scale multilingual informal--formal pairs~\citep{wu2022autoformalization,
azerbayev2023proofnet,jiang2023mma}.  A recent survey organizes this rapidly
growing area by mathematical domain, model, data, and evaluation
strategy~\citep{weng2025autoformalizationsurvey}.

Recent work moves from isolated statements toward complete proofs and research
libraries.  ProofFlow first recovers a directed dependency graph from an
informal proof and formalizes its steps as intermediate lemmas, explicitly
measuring semantic and structural fidelity~\citep{cabral2025proofflow}.  Beyond
the Library introduces an agentic pipeline that can add definitions and
auxiliary lemmas missing from Mathlib when formalizing research papers, and
TOMAP concentrates test-time refinement on the decomposition that supplies
claims, assumptions, and dependencies to downstream formalizer and prover
agents~\citep{moakhar2026beyond,liu2026tomap}.  FormalRx diagnoses semantic
misalignment by error type and location rather than returning only a binary
translation score, while theory-level autoformalization argues that useful
formalization must ultimately construct coherent libraries of definitions,
lemmas, theorems, and their interdependencies~\citep{wang2026formalrx,
min2026theorylevel}.  LeanMarathon addresses long-horizon research-paper
formalization through an evolving Lean blueprint that serves as both proof
skeleton and natural-language proof graph.  Four specialized agents stabilize
target fidelity and discharge the resulting dependency graph from its leaves in
parallel, producing complete Lean formalizations of seven target theorems from
recent research papers~\citep{zhang2026leanmarathon}.  These graph- and
theory-level perspectives are close to \VALG's emphasis on compositional
structure.  Their input, however, is an existing informal statement or proof to
be translated faithfully; \VALG uses a dependency graph earlier, while
constructing and revising the theorem itself, and leaves proof-assistant
formalization as a separate verification layer.

\subsection{Benchmarks for mathematical and research agents}

Formal benchmarks such as miniF2F, ProofNet, and PutnamBench evaluate proving or
formalizing fixed competition and undergraduate problems~\citep{zheng2021minif2f,
azerbayev2023proofnet,tsoukalas2024putnambench}.  Research-level evaluations
instead use expert-proposed open problems, as in Aletheia's FirstProof
benchmark~\citep{feng2026aletheia}, or extract larger problem collections from
the mathematical literature.  RMA evaluates an agentic system on
research-level problems, while ResearchMath-14K collects 14,056 questions and
documents recurrent failures including fabricated references, non-attempts, and
substitution of a narrower problem~\citep{zhao2026rma,son2026researchmath}.
These failures motivate evaluating both mathematical validity and fidelity to
the supplied research question.

Research-agent benchmarks cover a broader executable workflow.  MLAgentBench
evaluates iterative model training, debugging, and experimental improvement;
MLE-bench expands this setting to 75 offline Kaggle competitions; and RE-Bench
compares agents with human experts under different time budgets~\citep{
huang2023mlagentbench,chan2024mlebench,wijk2024rebench}.  Other suites begin
from published scientific work: CORE-Bench tests computational reproducibility,
PaperBench asks agents to replicate complete AI papers, and ScienceAgentBench
isolates expert-validated scientific coding tasks~\citep{siegel2024corebench,
starace2025paperbench,chen2025scienceagentbench}.  These benchmarks measure
experimentation, research engineering, reproducibility, and data analysis under
fixed task specifications.  Our case studies address a complementary question:
whether an agent can develop a theorem and, when necessary, a related
formulation while keeping each outcome mathematically tied to the original
ML-theory problem.

\section{\VALG: ML Theory Research Agent}
\label{sec:workflow}
\VALG develops theorem candidates from an ML-theory research direction or question. Starting from a source problem, it maps the relevant literature, identifies distinct perspectives, develops mechanism-level ideas, formalizes each viable idea as a precise setting and goal, constructs and reviews a proof from global structure to local derivations, and reports every accepted theorem relative to the source problem. A run may return zero, one, or several accepted theorem candidates.

%\subsection{Overall Workflow}
\VALG divides theorem development into a pre-proof stage for formalizing the problem setup and a proof-review stage for constructing and checking the proof. Dedicated revision loops operate at both stages because local derivations, proof structure, and theorem formulation require different repairs. Figure~\ref{fig:overall-workflow} summarizes the complete workflow.

\begin{figure}[!t]
  \centering
  \resizebox{\textwidth}{!}{\input{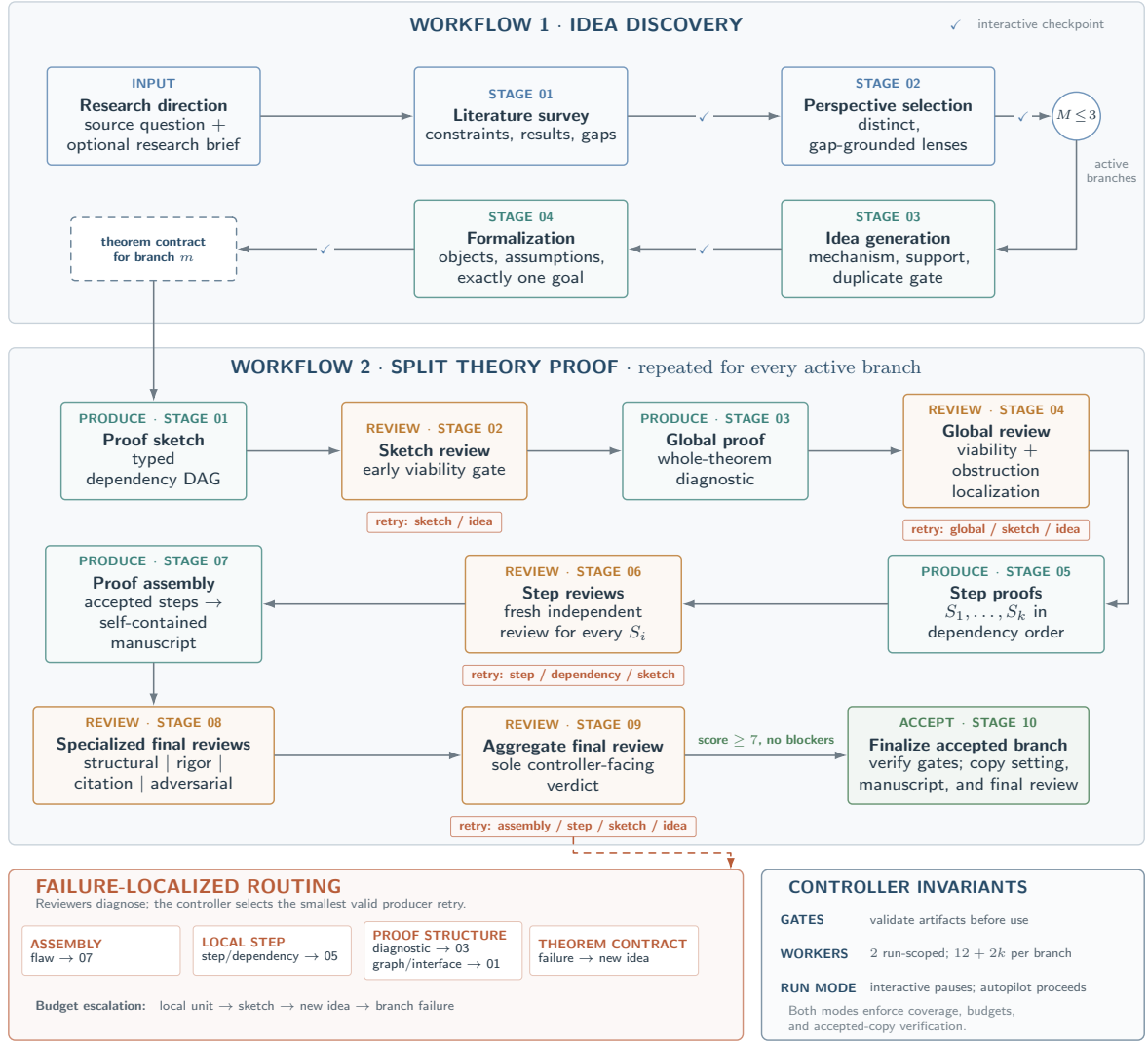}}
  \caption{The \VALG workflow. Workflow~1 creates active perspective branches; Workflow~2 is shown for one branch. Teal, amber, and green denote producers, independent reviewers, and accepted output, respectively. Solid arrows show validated forward flow. A finalized branch does not stop the remaining branches.}
  \label{fig:overall-workflow}
\end{figure}

\subsection{Workflow 1: Pre-Proof Stage}
Theorem proving in pure mathematics often begins with well-defined problem setup, including formalized assumptions, mathematical objects and targets. However, an ML-theory research question is often less settled. It may specify a learning phenomenon but \textit{leave open which model subclass, data distribution, algorithmic mechanism, performance metric, or asymptotic regime makes a theorem both true and informative}. Workflow~1 therefore treats problem formulation as an explicit research task rather than an implicit prerequisite for proof search.

The workflow begins with a source question and, when available, an accompanying research brief. Guided by the literature survey, it organizes exploration around a two-level \textbf{perspective-idea hierarchy}. A perspective provides a broad research lens, whereas an idea fully specifies a concrete problem setup sufficient for downstream formalization and theorem proving. Steps of Workflow 1 are summarized as follows.

\begin{enumerate}[leftmargin=*]
    \item \textbf{Literature survey.} The literature worker maps the relevant settings, results, proof techniques, and evidence-supported gaps. Different from conventional literature survey in auto research, ML-theory research 
    \begin{tcolorbox}[valg-note]
    \begin{itemize}[leftmargin=*,nosep]
        \item should separate empirical evidence from theoretical foundations to \textcolor{MLTblue}{discover research gaps in explaining empirical phenomenon};
        \item should determine applicable theoretical testbeds to \textcolor{MLTblue}{study theoretical research question}.
    \end{itemize}
    \end{tcolorbox}
    We therefore separate existing literature into direct theoretical sources, foundational theoretical frameworks and empirical practice, while preserving the constraints that define the source direction.
    \item \textbf{Two-level perspective-idea structure.}
    \begin{tcolorbox}[valg-note]
    \begin{itemize}[leftmargin=*]
        \item \textbf{Breadth-focus perspective selection.} The perspective selector to select a small set of distinct, gap-grounded perspectives and launches them as parallel research branches, where a \emph{perspective} is defined as a coherent, literature-supported research lens, represented by the normalized tuple
        \[
            (\text{analysis target},\ \text{model class},\ \text{data assumption},\ \text{regime},\ \text{algorithm}).
        \]
        Such characterization for perspective distinguishes a meaningful research direction, such as the desired theoretical guarantee, model and data setting, learning regime, or algorithmic family. It needs not fix every aspect of the eventual theorem; rather, it defines a broad search region and leaves appropriate choices open for refinement during idea generation.
        \item \textbf{Depth-focus idea generator.} Within each perspective branch, an distinct idea generator proposes a concrete mechanism and candidate result \textbf{using the upstream perspective as an explicit anchor}. Ideas that are unsupported by the literature or duplicate ideas already proposed in another branch are rejected. \textbf{Each generated idea must be a field-wise specialization of its upstream perspective}: it may make the analysis target, model class, data assumption, regime, or algorithm more specific, but it may not broaden any of these choices or contradict the perspective.
    \end{itemize}
    \textcolor{MLTnavy}{This separation keeps exploration broad without spending all proof effort on a single interpretation, while giving each surviving branch a focused mechanism and theorem target.}
    \end{tcolorbox}
    \item \textbf{Formalization.} For each viable, nonduplicate idea, the formalizer fixes the notation, primitive assumptions, quantifiers, regime, and exactly one mathematical goal. The resulting setting is the theorem contract passed to Workflow~2.
\end{enumerate}

\paragraph{Human-expert interactive mode.} We remark that early failures are expected:
\begin{itemize}
    \item research gaps from literature survey may not be valuable in research level;
    \item perspectives may be too narrow, too broad, or redundant;
    \item ideas may depend on assumptions that make the resulting theorem unacceptable.
\end{itemize}
Because these decisions are scientific as well as formal, Workflow~1 places a human-expert checkpoint after each stage. In interactive mode, the expert may approve the artifact, edit it, or request a new search or generation with human-expert feedback. Autopilot mode proceeds with approval by default, but retains the same source-direction fidelity and branch-coverage constraints. In either mode, only a checked and formalized branch enters proof development.

\subsection{Workflow 2: Proof-Review Stage}
Following Workflow 1, Workflow 2 develops mathematical proofs based on the provided formalized setting and goal. Our initial approach relied on monolithic proof generation, which entangles theorem-level architecture, local derivations, and final exposition. This method suffered from one severe limitation: \textit{the resulting proofs were too compact to effectively review or understand}.

Drawing on recent progress in natural-language mathematical reasoning, where systems typically generate a proof plan, tackle each step individually, and then assemble the final proof \citep{an2026qed,ju2026rethlas,feng2026aletheia}, we introduced a similar \textbf{sketch-step-assembly} decomposition to make local claims individually addressable.
This decomposition, however, revealed a second gap: \textit{a structurally coherent dependency graph does not guarantee that every node is derivable from its assigned inputs, nor does it ensure that its conclusion possesses the precise form and quantitative strength required downstream}. To resolve this, we inserted a theorem-level global diagnostic between the sketching and local proof phases, ultimately structuring Workflow 2 as a four-stage \textbf{sketch-global-step-assembly pipeline}.

\subsubsection{Proof Stage}
\begin{enumerate}[leftmargin=*]
    \item \textbf{Proof sketch.} The sketch worker represents the proposed argument as a \textbf{directed acyclic dependency graph}. Source nodes encode primitive assumptions, internal nodes encode lemma-sized claims, and the unique sink is the target theorem. Each node records its exact claim, dependencies, permitted assumptions, intended proof tool, and required output interface.
    \item \textbf{Global proof.} The global-proof worker \textbf{expands the accepted graph into a whole-theorem diagnostic}. For every node, it traces the available inputs, identifies the mechanism or cited result that could establish the claim, and checks whether the output has the form and strength required by downstream nodes. At theorem level, it audits quantitative dependence, probability and convergence modes, object compatibility, closure arguments, boundary behavior, and the composition of local interfaces. 
    \begin{tcolorbox}[valg-note]
    Global proof is crucial as \textcolor{MLTnavy}{proof dependency graph consistency does not ensure theorem-level composability and local-derivation feasibility}.
    For example, an upstream guarantee for a surrogate object may not imply the guarantee required for the original target downstream. The global diagnostic identifies these cross-step mismatches, specifies each step's admissible inputs and required output, and guides step workers for full local derivations.
    \end{tcolorbox}
    \item \textbf{Proof step.} A dedicated step worker converts one graph obligation into proof evidence using the formal setting, accepted dependency results, and the global diagnostic as context. It may introduce local lemmas, but must state and prove them before deriving the exact assigned claim, without strengthening assumptions or weakening the required output. 
    \item \textbf{Proof assembly.}
    Once all required steps are accepted, the assembler reconciles notation and composes their claims into a self-contained \LaTeX{} theorem manuscript. This stage verifies that the accepted local interfaces yield the stated theorem in one coherent presentation.
\end{enumerate}
\subsubsection{Review Stage}
We would like to highlight a key difference between Workflow 1 and Workflow 2 in \textit{review mechanisms} because they govern different kinds of decisions.
\begin{tcolorbox}[
    valg-comparison,
    title={Human-expert checkpoints versus agent reviewers}]
    \begin{itemize}[leftmargin=*]
    \item Workflow~1 uses \textbf{\textcolor{MLTblue}{human-expert checkpoints}} because evaluating literature gaps and selecting research perspectives and theorem formulations require \textbf{\textcolor{MLTblue}{open-ended judgments about research value}} that depend strongly on human expertise.
    \item Workflow~2 begins from a \textbf{\textcolor{MLTblue}{fixed theorem contract}}. This contract allows \textbf{\textcolor{MLTblue}{reviewers distinct from the producers}} to assess each proof artifact against explicit, stage-specific obligations of goal alignment, assumption fidelity, logical soundness, and derivational rigor before the controller authorizes its downstream use. Because both the target and the evaluation criteria are fixed, this contract-based review is better suited to specialized agent reviewers than the open-ended decisions in Workflow~1.
\end{itemize}
\end{tcolorbox}
Hence, within Workflow~2, we place a distinct sketch reviewer checks proof architecture after the sketch worker, a distinct global reviewer checks theorem-level feasibility after the global proof worker and a distinct step reviewer checks local derivations after the proof step worker. This separation prevents producer self-validation and exposes defects at the same level of abstraction at which they arise. Below we show the key rubrics of both sketch and step reviewers.
\begin{tcolorbox}[
    valg-rubric,
    title={Key rubrics}]
\begin{itemize}[leftmargin=*]
    \item \textbf{Proof-sketch review} (\skill{proof-sketch-review}). 
    \begin{enumerate}[leftmargin=*,nosep]
        \item \textbf{Goal fidelity.} Check exact-goal or target-spec alignment, including quantifiers, regimes, probability or convergence modes, normalization, and exposed rate dependence.
        \item \textbf{Graph and coverage.} Verify an acyclic graph of stable, lemma-sized steps with legal earlier dependencies, allowed assumptions, intended proof tools, output targets, and explicit blockers for unresolved high-risk obligations.
        \item \textbf{Provenance and interfaces.} Trace assumptions, derived invariants, and cited tools to legal sources; trace every generated output from producer to consumers and through any required current-notation or residual-to-target bridge.
        \item \textbf{Viability gates.} Check mechanism witnesses, baseline and entry behavior.
    \end{enumerate}
    \item \textbf{Proof-step review} (\skill{proof-step-review}). 
    \begin{enumerate}[leftmargin=*,nosep]
        \item \textbf{Local-unit audit.} Check every local lemma statement and derivation against the exact sketch-row claim, allowed assumptions, and accepted dependencies.
        \item \textbf{Hidden-claim and provenance audit.} Scan for independent subclaims; trace assumptions, constants, rates, events, and notation; restate cited results in current notation with all hypotheses discharged; stress quantifiers, modes, and boundary cases.
        \item \textbf{Target assembly.} Verify that named local results, checked citations, and accepted dependencies jointly establish the exact target claim without strengthening assumptions or weakening the required output.
    \end{enumerate}
\end{itemize}
\end{tcolorbox}

For assembly worker, which \textbf{is the first stage where all accepted components, notation, citations, and connecting implications must jointly establish the exact theorem}, however, the substantial review task motivates us to separate review into independent, multi-angle reviews \citep{an2026qed}: dependency and coverage errors (structural), invalid derivations (rigor), misapplied external results (citation), and hidden boundary-case failures (adversarial). Each reviewer is therefore tied to a specific assembly-level risk, rather than being an arbitrary addition. Finally, an \textit{aggregate reviewer} reconciles them into the sole controller-facing final-review verdict and identifies the smallest admissible repair target.
\begin{tcolorbox}[
    valg-rubric,
    title={Key rubrics}]
\begin{itemize}[leftmargin=*]
    \item \textbf{Structural review} (\skill{proof-review-structural}). 
    \begin{enumerate}[leftmargin=*,nosep]
        \item Check the assembled claim against the setting, close every dependency, trace every required sketch step into an accepted proof and the public appendix.
        \item Reject independent mathematics introduced during assembly.
        \item The public \LaTeX{} must remain self-contained, paper-ready, and explicit about its assumptions and theorem-style references.
    \end{enumerate}
    \item \textbf{Rigor review} (\skill{proof-review-rigor}). 
    \begin{enumerate}[leftmargin=*,nosep]
        \item Inspect the actual derivations rather than labels or environment counts.
        \item Audit quantifier order, constants and parameter dependence, probability and convergence modes, assumption provenance, explicit-rate specialization bridges, interchanges, and boundary cases.
        \item Compare every used step's proof obligations with its appendix proof body.
    \end{enumerate}
    \item \textbf{Citation review} (\skill{proof-review-citation}). 
    \begin{enumerate}[leftmargin=*,nosep]
        \item Verify that each cited or internal result exists and supports the exact conclusion used, translate its objects and notation into the current setting, and discharge every hypothesis in the claimed regime.
        \item Check BibTeX keys and internal label--reference pairs.
    \end{enumerate}
    \item \textbf{Adversarial review} (\skill{proof-review-adversarial}). 
    \begin{enumerate}[leftmargin=*,nosep]
        \item Attack the weakest claims with concrete assumption-minimal, boundary, degenerate, and extreme regimes.
        \item Test hidden assumption strengthening, unsupported scope or convergence-mode upgrades, unproved derived invariants, and failed baseline reductions.
        \item Treat verified breaks and unresolved high-risk candidate counterexamples as blocking.
    \end{enumerate}
\end{itemize}
\end{tcolorbox}

\subsection{Revision Loops}
\label{sec: revision loop}
A failed proof attempt may expose a defect in a local derivation, a dependency interface, the proof architecture, or the theorem formulation itself. In ML-theory research, it may also show that the initial problem setup requires an additional condition or a revised assumption before a valid theorem is possible. \VALG therefore uses \textbf{a hierarchy of revision loops} that routes each diagnosis to the smallest stage capable of repairing it.
% \dechen{illustrate hierarchy loop, from local derivations to idea (problem setup), illustrate loop budget}
This hierarchical structure, along with its corresponding review routing mechanisms, is organized as follows:
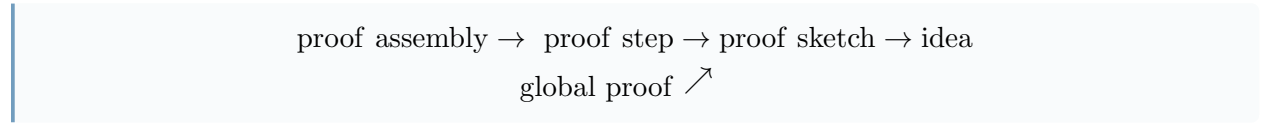
\begin{tcolorbox}[valg-note]
    \begin{equation*}
    \begin{aligned}
        \mathrm{proof\ assembly} \to\ \mathrm{proof\ step}& \to \mathrm{proof\ sketch} \to \mathrm{idea}\\
        \mathrm{global\ proof}&\ \raisebox{3pt}{$\nearrow$}
    \end{aligned}
\end{equation*}
\end{tcolorbox}
\begin{itemize}[leftmargin=*]
    \item \textbf{proof-sketch review}: \texttt{ACCEPTED}, \texttt{REVISE\_SKETCH}, or \texttt{IDEA\_FAIL};
    \item \textbf{global-proof review}: \texttt{ACCEPTED}, \texttt{REVISE\_GLOBAL}, \texttt{REVISE\_SKETCH}, or \texttt{IDEA\_FAIL};
    \item \textbf{proof-step review}: \texttt{ACCEPTED}, \texttt{REVISE\_STEP}, \texttt{BLOCKED\_BY\_DEPENDENCY}, or \texttt{REVISE\_SKETCH};
    \item \textbf{aggregate final review}: \texttt{ACCEPTED}, \texttt{PROOF\_ASSEMBLY\_FLAW},\
    \texttt{PROOF\_STEP\_FLAW}, \texttt{PROOF\_SKETCH\_FLAW}, or \texttt{IDEA\_FAIL}.
\end{itemize}
In every case, the failed review identifies the smallest repair target, the controller selects the responsible stage subject to its retry budget \footnote{If a stage's retry budget is exhausted, the controller escalates the issue by routing it to a higher-level stage that still has remaining budget.}, and the selected producer uses the accepted upstream context, the failed attempt, and the validated diagnosis to change only the implicated part. Figure~\ref{fig:revision-loops} summarizes this controlled revision cycle.

\begin{figure}[!t]
  \centering
  \resizebox{\textwidth}{!}{\definecolor{rlInk}{HTML}{243746}
\definecolor{rlNavy}{HTML}{264A67}
\definecolor{rlBlue}{HTML}{4A78A8}
\definecolor{rlTeal}{HTML}{3B817B}
\definecolor{rlAmber}{HTML}{B9782D}
\definecolor{rlGreen}{HTML}{4B8058}
\definecolor{rlRust}{HTML}{B85C38}
\definecolor{rlLine}{HTML}{C8D2DA}

\begin{tikzpicture}[
  font=\sffamily\footnotesize,
  text=rlInk,
  line cap=round,
  line join=round,
  stage/.style={
    draw=rlLine,
    rounded corners=2pt,
    line width=0.65pt,
    fill=white,
    align=left,
    text width=3.85cm,
    minimum height=1.72cm,
    inner sep=6pt,
    execute at begin node={\raggedright\hyphenpenalty=10000\exhyphenpenalty=10000}
  },
  current/.style={stage,draw=rlBlue!62,fill=rlBlue!3},
  review/.style={stage,draw=rlAmber!72,fill=rlAmber!4},
  control/.style={stage,draw=rlNavy!66,fill=rlBlue!3},
  produce/.style={stage,draw=rlTeal!72,fill=rlTeal!4},
  terminal/.style={
    draw=rlGreen!70,
    rounded corners=2pt,
    line width=0.65pt,
    fill=rlGreen!5,
    align=center,
    text width=3.15cm,
    minimum height=1.05cm,
    inner sep=5pt
  },
  flow/.style={
    -{Latex[length=2mm,width=1.35mm]},
    draw=rlInk!62,
    line width=0.72pt
  },
  retry/.style={
    -{Latex[length=2.3mm,width=1.5mm]},
    draw=rlRust,
    dashed,
    line width=0.8pt
  },
  edgelabel/.style={
    fill=white,
    inner sep=1.5pt,
    font=\sffamily\footnotesize,
    text=rlInk!70
  }
]
  \node[current] (attempt) at (2.15,0) {
    {\scriptsize\bfseries\color{rlBlue}1 \; CURRENT ATTEMPT}\\[2pt]
    {\bfseries Attempt $a$}\\[1pt]
    Evaluate one stage result against its stated target and accepted inputs.
  };

  \node[review] (diagnosis) at (8.25,0) {
    {\scriptsize\bfseries\color{rlAmber}2 \; INDEPENDENT REVIEW}\\[2pt]
    {\bfseries Diagnose the failure}\\[1pt]
    Identify the failure type and the smallest mathematical object that must change.
  };

  \node[control] (controller) at (14.35,0) {
    {\scriptsize\bfseries\color{rlNavy}3 \; CONTROLLER ROUTING}\\[2pt]
    {\bfseries Select the repair level}\\[1pt]
    Validate the diagnosis and route one scoped retry to the responsible stage.
  };

  \node[produce] (revision) at (14.35,-2.85) {
    {\scriptsize\bfseries\color{rlTeal}4 \; SELECTED PRODUCER}\\[2pt]
    {\bfseries Construct Attempt $a+1$}\\[1pt]
    Revise only the implicated mathematical component.
  };

  \node[review] (fresh) at (8.25,-2.85) {
    {\scriptsize\bfseries\color{rlAmber}5 \; FRESH REVIEW}\\[2pt]
    {\bfseries Check the revision}\\[1pt]
    Repeat the required independent review or human checkpoint before downstream use.
  };

  \node[terminal] (accepted) at (2.15,-2.85) {
    {\bfseries\color{rlGreen}Accepted}\\[1pt]
    Advance the revised result downstream
  };

  \draw[flow] (attempt) -- (diagnosis);
  \draw[flow] (diagnosis) -- (controller);
  \draw[flow] (controller) -- (revision);
  \draw[flow] (revision) -- (fresh);
  \draw[flow,draw=rlGreen] (fresh) -- (accepted);

  \draw[retry,rounded corners=4pt]
    (fresh.south) -- ++(0,-0.64)
    -- node[edgelabel,below,text=rlRust]{blocking: return to Step 3} ++(9.15,0)
    |- (controller.east);
\end{tikzpicture}}
  \caption{Controlled revision in \VALG. An independent review identifies the smallest repair target, the controller routes the diagnosis to the responsible stage, and the selected producer revises only the implicated part. The revised attempt must pass a fresh independent review or human checkpoint before downstream use.}
  \label{fig:revision-loops}
\end{figure}
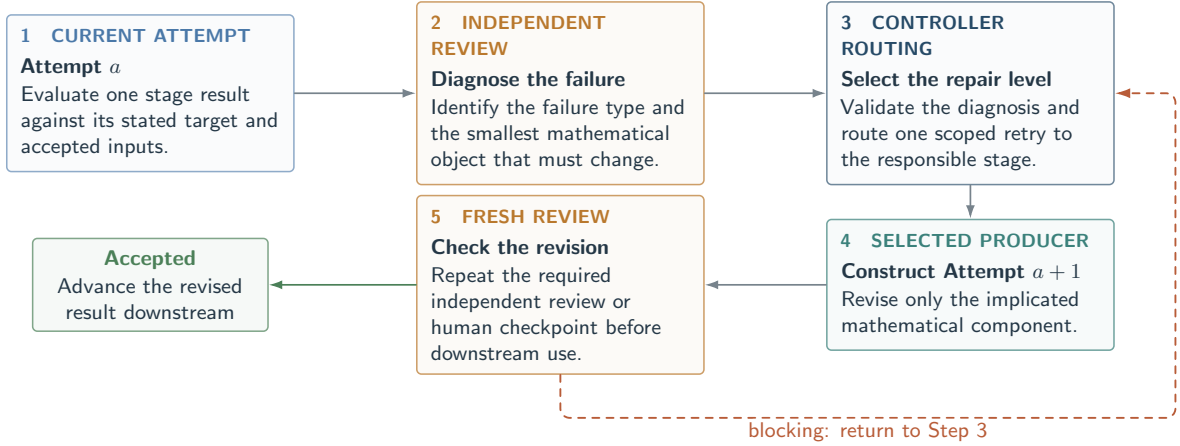

\newpage
\section{COLT 2026 Case Studies}
\label{sec:evaluation}

We evaluate \VALG on nine subproblems drawn from five COLT 2026 open-problem papers~\citep{lau2026interaction,arvanitakis2026als,balcan2026online,nissim2026private,feldman2026deep}. Each subproblem is run independently from a research brief extracted from its source description, using GPT-5.6-sol at maximum reasoning effort. Finalized theory candidates have been checked through independent multi-perspective LLM reviewing and roughly audited by human\footnote{The correctness of the proof may also need to be verified by more experts, especially the authors of these open-problem papers. }.

Table~\ref{tab:perspective-search-summary} reports every perspective branch represented in the archived runs. Within a run, \VALG may develop several perspective branches (with maximum budget 3) in parallel. For each perspective branch, the idea generator prioritizes a source-faithful candidate aimed at fully resolving the subproblem. If no such ideas are viable, it will then attempt to propose an idea with partial-progress type, i.e., a relaxed problem setup. After the first idea is proposed, subsequent idea variants are opened only after explicit controller routing, such as an idea-level review failure, or escalation after lower-level retry budgets are exhausted. See Section \ref{sec: revision loop} for illustration of the revision rules.

\begin{table}[!t]
  \centering
  \small
  \captionsetup{
    font=small,
    labelfont={bf,sf,color=MLTnavy},
    justification=raggedright,
    singlelinecheck=false,
    skip=7pt
  }
  \setlength{\tabcolsep}{4pt}
  \renewcommand{\arraystretch}{1.12}
  \arrayrulecolor{MLTline}

  \caption{Perspective-level search and evaluation summary.}
  \label{tab:perspective-search-summary}

  \begin{tabular}{@{}>{\raggedright\arraybackslash}p{.30\linewidth}
  >{\raggedright\arraybackslash}p{.13\linewidth}
  >{\centering\arraybackslash}p{.06\linewidth}
  >{\centering\arraybackslash}p{.06\linewidth}
  >{\centering\arraybackslash}p{.11\linewidth}
  >{\centering\arraybackslash}p{.075\linewidth}
  >{\centering\arraybackslash}p{.09\linewidth}
  >{\raggedleft\arraybackslash}p{.05\linewidth}@{}}

  \toprule
  \rowcolor{MLTblue!7}
  \textcolor{MLTnavy}{\sffamily\bfseries Problem} &
  \textcolor{MLTnavy}{\sffamily\bfseries Subproblem} &
  \textcolor{MLTnavy}{\sffamily\bfseries Persp.} &
  \textcolor{MLTnavy}{\sffamily\bfseries\shortstack{Ideas\\tried}} &
  \textcolor{MLTnavy}{\sffamily\bfseries\shortstack{Branch\\outcome}} &
  \textcolor{MLTnavy}{\sffamily\bfseries\shortstack{Elapsed\\(h)}} &
  \textcolor{MLTnavy}{\sffamily\bfseries\shortstack{Progress}} &
  \textcolor{MLTnavy}{\sffamily\bfseries\shortstack{Rank}}\\
  \midrule

  \multirow{5}{=}{\textcolor{MLTnavy}{\sffamily\bfseries P1: Tensor ALS}} &
  \multirow{2}{=}{Upper bound} & 1 & 7 & \tableidea{7} & 76 & 6.50 & 1 \\
   & & 2 & 14 & Exhausted & -- & 0.00 & 2 \\
   & \multirow{3}{=}{Lower bound} & 1 & 2 & \tableidea{2} & 8 & 2.25 & 2 \\
   & & 2 & 2 & \tableidea{2} & 11 & 3.00 & 1 \\
   & & 3 & 3 & \tableidea{3} & 38 & 3.00 & 3 \\

  \midrule

  \multirow{3}{=}{\textcolor{MLTnavy}{\sffamily\bfseries P2: 1-bit mean estimation}} &
  \multirow{3}{=}{Non-adaptive} & 1 & 1 & \tableidea{1} & 13 & \textbf{10.00} & 2 \\
   & & 2 & 2 & Interrupted & -- & 0.00 & 3 \\
   & & 3 & 1 & \tableidea{1} & 9 & \textbf{10.00} & 1 \\

  \midrule

  \multirow{6}{=}{\textcolor{MLTnavy}{\sffamily\bfseries P3: Deep vs. linear learning}} &
  \multirow{3}{=}{SGD} & 1 & 3 & \tableidea{3} & 12 & 2.75 & 2 \\
   & & 2 & 2 & \tableidea{2} & 10 & 2.25 & 3 \\
   & & 3 & 2 & \tableidea{2} & 12 & 5.25 & 1 \\
   & \multirow{3}{=}{SQ} & 1 & 3 & \tableidea{3} & 15 & 2.75 & 2 \\
   & & 2 & 2 & \tableidea{2} & 8 & 3.00 & 1 \\
   & & 3 & 2 & \tableidea{2} & 9 & 2.25 & 3 \\

  \midrule

  \multirow{6}{=}{\textcolor{MLTnavy}{\sffamily\bfseries P4: Private PAC learning}} &
  \multirow{3}{=}{Sample complexity} & 1 & 3 & \tableidea{3} & 48 & 6.50 & 1 \\
   & & 2 & 4 & \tableidea{4} & 91 & 6.25 & 2 \\
   & & 3 & 1 & \tableidea{1} & 28 & 5.25 & 3 \\
   & \multirow{3}{=}{Class existence} & 1 & 5 & \tableidea{5} & 28 & 3.00 & 2 \\
   & & 2 & 2 & \tableidea{2} & 25 & 5.50 & 1 \\
   & & 3 & 6 & Exhausted & -- & 0.00 & 3 \\

  \midrule

  \multirow{5}{=}{\textcolor{MLTnavy}{\sffamily\bfseries P5: Online optimization}} &
  \multirow{2}{=}{Polynomial} & 1 & 1 & \tableidea{1} & 12 & 5.75 & 2 \\
   & & 2 & 1 & \tableidea{1} & 24 & 7.00 & 1 \\
   & \multirow{3}{=}{Pfaffian} & 1 & 1 & \tableidea{1} & 30 & \textbf{10.00} & 1 \\
   & & 2 & 1 & \tableidea{1} & 40 & \textbf{10.00} & 2 \\
   & & 3 & 1 & \tableidea{1} & 31 & 6.75 & 3 \\

  \bottomrule
  \end{tabular}

  \vspace{0.4em}

  \begin{minipage}{\linewidth}
  \footnotesize
  \color{MLTink!70}
  Progress \(P\in[0,10]\) is evaluated against the original paper's
  subproblem by a distinct agent. Rank is computed within each subproblem
  from \(0.40\,{\times}\) progress,
  \(0.40\,{\times}\) mathematical soundness, and
  \(0.20\,{\times}\) technical novelty by a distinct agent.
  \emph{Exhausted} indicates that the branch consumed its configured
  idea-variant budget without finalizing a theory, whereas
  \emph{Interrupted} indicates that execution ended before finalization by
  human. A dash denotes unavailable elapsed time.
  \end{minipage}

  \arrayrulecolor{black}
  \end{table}

We would like to remark that progress is computed as \(P=\min\{C+B+H,\text{applicable cap}\}\), where “applicable cap” is a hard ceiling triggered by a fundamental mismatch with the original problem. For example, $P\leq 3$ if the theorem assumes a central unresolved property, $P\leq 7$ if it proves only one direction of a characterization or omits a central quantifier, regime, or rate. \(C\in[0,4]\) measures closure of the paper-level target contract, \(B\in[0,3]\) measures improvement over the paper's baseline, and \(H\in[0,3]\) measures how much of the central open burden is discharged. A score of \(P=10\) denotes exact target coverage. Regarding mathematical soundness, \(S=\min\{M,D,A\}\) is a weakest-link assessment of intrinsic mathematical validity, verification of theorem-critical dependencies, and audit completeness. For technical novelty, \(N\) measures the originality and proof-critical role of the technical mechanism relative to prior work.

Results show that two of nine subproblems are fully solved while the remaining seven subproblems are partially solved. For each subproblem, we report the most informative branches and state their progress relative to the corresponding source problem. Proof details are available at \url{https://github.com/DechenZhang/VALG-ML-Theory-Agent/tree/main/case-studies/colt-2026}.

\subsection{How much overparametrization is needed for ALS in tensor decomposition?}
\label{sec:case-als}

\begin{openquestion}
Let
\[
T=\sum_{j=1}^{r}a_j\otimes b_j\otimes c_j\in\mathbb R^{n\times n\times n}
\]
where $a_j, b_j$ and $c_j$ are obtained by adding mutually independent
$\mathcal N(0,\rho^2 I_n/n)$ perturbations, with
$\rho=1/\operatorname{poly}(r)$, to deterministic factor columns. Consider an optimization problem minimizing the least-squares objective
\[
\min_{x_i,y_i,z_i \in \mathbb R^n,\ i\in [k]} \left\|T-\sum_{i=1}^k x_i\otimes y_i \otimes z_i\right\|_F^2.
\]
The open question of \citet{arvanitakis2026als} asks \textit{what is the smallest value of $k=k(r)$ so that alternating least squares (ALS) or another iterative algorithm such as GD converges from random initialization to the global optimum with high probability.}
\end{openquestion}

\subsubsection{Subproblem 1: Upper Bound}

\begin{openquestion}
Does some iterative method, with $r<k=o(r^2)$ components, return in
$\operatorname{poly}(n,r,\log(1/\epsilon))$ time a decomposition satisfying
\[
 \left\|T-\sum_{i=1}^k x_i\otimes y_i\otimes z_i\right\|_F
 \le \epsilon\|T\|_F
\]
with high probability over the smoothed instance?
\end{openquestion}

\paragraph{\underline{Perspective 1.}}

\paragraph{Formalized setting and preliminaries.}
\label{par:als-upper-setting}
Fix integers $r\ge3,n$ and let $q_*=1/4096$.  For deterministic
$\bar A,\bar B,\bar C\in\mathbb R^{n\times r}$ with nonzero columns, define
$\bar u_j=\bar a_j/\|\bar a_j\|_2$ and cyclically $\bar v_j,\bar w_j$;
write $\bar U=[\bar u_j]$, $\bar V=[\bar v_j]$, $\bar W=[\bar w_j]$ and
$\bar\lambda_j=\|\bar a_j\|_2\|\bar b_j\|_2\|\bar c_j\|_2$.  For a
unit-column matrix $M=[m_j]$, write
\[
 q(M)=\max_j\sum_{\ell\ne j}|\langle m_j,m_\ell\rangle|,
 \qquad \bar q=\max_{M\in\{\bar U,\bar V,\bar W\}}q(M).
\]
Independently over columns and modes, draw
$g_j^{(A)},g_j^{(B)},g_j^{(C)}\sim\mathcal N(0,\rho^2I_n/n)$, set
$a_j=\bar a_j+g_j^{(A)}$ and cyclically $b_j,c_j$, and form
\[
 T=\sum_{j=1}^r a_j\otimes b_j\otimes c_j
   =\sum_{j=1}^r\lambda_j u_j\otimes v_j\otimes w_j,
\]
where $U=[u_j]$, $V=[v_j]$, $W=[w_j]$ are the normalized realized factors and
$\lambda_j=\|a_j\|_2\|b_j\|_2\|c_j\|_2$.  Put
$q_{\rm real}=\max\{q(U),q(V),q(W)\}$ and
$\Gamma=\max_j\lambda_j/\min_j\lambda_j$.

The proposed procedure uses
\[
 k=\left\lceil C_{\rm rank}r^{5/3}(\log r)^{5/2}\right\rceil,
 \quad L_{\rm burn}=\lceil C_{\rm burn}\log r\rceil,
 \quad L_{\rm cert}=\lceil C_{\rm cert}\log r\rceil,
 \quad \tau_r=\frac{q_*^2}{10^4r}.
\]
For each slot $i$ and mode $M\in\{U,V,W\}$, draw mutually independent raw
vectors $\xi_i^{(M)}\sim\mathcal N(0,I_n)$, independently of the smoothing,
and initialize $(p_i^0,q_i^0,s_i^0)$ by normalizing the corresponding raw
triple.
For $h=(p,q,s)$, it repeatedly makes the simultaneous old-state Jacobi commit
\[
 \mathcal J(h)=\left(
 \frac{T(\cdot,q,s)}{\|T(\cdot,q,s)\|_2},
 \frac{T(p,\cdot,s)}{\|T(p,\cdot,s)\|_2},
 \frac{T(p,q,\cdot)}{\|T(p,q,\cdot)\|_2}\right).
\]
After $L_{\rm burn}$ commits it inspects the states through
$t=L_{\rm burn}+L_{\rm cert}$ and stores the first one for which
$\max_M\min_{\sigma\in\{\pm1\}}\|h_M-\sigma\mathcal J_M(h)\|_2\le\tau_r$.
Only states with nonzero contractions and a valid certificate proceed.  Certified triples
are filtered at $0.85$ of the largest score
$|\langle T,p\otimes q\otimes s\rangle|$.  Form a graph on the remaining
triples, joining two vertices when their modewise absolute correlations are
all at least $1-64q_*$.  The proposal is accepted only if this graph has
exactly $r$ connected components.  The minimum-displacement representative
of each component, with score as tie-break, gives an equal-norm best-scalar
seed.
For representative $(p_a,q_a,s_a)$ with
$\theta_a=\langle T,p_a\otimes q_a\otimes s_a\rangle\ne0$, this seed is
\[
 x_a^0=|\theta_a|^{1/3}p_a,\qquad
 y_a^0=|\theta_a|^{1/3}q_a,\qquad
 z_a^0=\operatorname{sgn}(\theta_a)|\theta_a|^{1/3}s_a.
\]
A zero score aborts the run.

Here $\odot$ denotes the columnwise Khatri--Rao product.  The three designs
$Z^0\odot Y^0$, $Z^0\odot X^0$, and $Y^0\odot X^0$ are
frozen before any solve.  Three
Moore--Penrose least-squares landing proposals are computed from that same
seed and committed synchronously.  The run aborts if any committed active
column is zero; otherwise the proposals are rebalanced once without changing
their rank-one products.  Cyclic $U/V/W$ ALS then updates the $r$ active columns; the other
$k-r$ columns stay zero.  Each run stops at relative residual $\epsilon$ or
after $\lceil C_{\rm stop}\log(8\kappa_0^2/\epsilon)\rceil$ sweeps.

For raw proposal vectors $\xi_i^{(M)}$, set
$m_{U,j}=u_j$, $m_{V,j}=v_j$, $m_{W,j}=w_j$ and
$Z_{ij}^{(M)}=\langle m_{M,j},\xi_i^{(M)}\rangle$ and
$t_r=\sqrt{(10/9)\log r}$.  The analysis uses the \textbf{target-slot event}
\[
 \mathcal E_{ij}=\bigcap_M\{t_r\le|Z_{ij}^{(M)}|\le t_r+t_r^{-1}\}
 \cap\!\bigcap_{\ell\ne j}\bigcap_{\{M,N\}}
 \left\{|Z_{i\ell}^{(M)}Z_{i\ell}^{(N)}|\le\frac{19}{18}\log r\right\},
\]
where $M,N$ range over distinct pairs from $\{U,V,W\}$.

\paragraph{Technical assumptions.}
\begin{assumption}[Bounded deterministic base scales]
\label{assump:als-upper-base-scale}
Every base-column norm lies in $[\kappa_0^{-1},\kappa_0]$, where
$1\le\kappa_0\le r^{d_\kappa}$ for a fixed finite $d_\kappa$.
\end{assumption}

\begin{assumption}[Cumulative Gram interference]
\label{assump:als-upper-cumulative-gram}
$\bar q\le q_*/4$.
\end{assumption}

\begin{assumption}[Near-balanced deterministic weights]
\label{assump:als-upper-weight-balance}
$\max_j\bar\lambda_j/\min_j\bar\lambda_j\le1+1/800$.
\end{assumption}

\begin{assumption}[Independent Gaussian smoothing]
\label{assump:als-upper-smoothing}
The $3r$ perturbations have the independent laws above, with
$0<\rho\le1$ and $\rho^{-1}\le r^{d_\rho}$ for a fixed finite $d_\rho$.
\end{assumption}

\begin{assumption}[Scale-aware smoothing and dimension margin]
\label{assump:als-upper-smoothing-margin}
For $\delta_{\rm sm}\in(0,1)$,
\[
 \kappa_0\rho\le\frac{q_*}{128},\qquad
 r(\kappa_0\rho+\kappa_0^2\rho^2)
 \sqrt{\frac{\log(9r^2/\delta_{\rm sm})}{n}}\le\frac{q_*}{32}.
\]
\end{assumption}

\begin{assumption}[Strictly subquadratic proposal rank]
\label{assump:als-upper-rank}
The displayed $k$ satisfies $r<k\le n$.
\end{assumption}

\begin{assumption}[Independent proposal and restart randomness]
\label{assump:als-upper-randomness}
Conditional on the once-drawn tensor, proposal triples are independent across
slots, modes, and completed runs, and independent of smoothing; restarts reuse
the tensor but not proposal randomness.
\end{assumption}

\begin{assumption}[Accuracy and separate confidence levels]
\label{assump:als-upper-confidence}
$0<\epsilon,\delta_{\rm sm},\delta_{\rm init}<1$, with the latter two
controlling the smoothed instance and conditional restarts, respectively.
\end{assumption}

\paragraph{Main theory.}
\begin{theorem}[Conditional strictly subquadratic recovery]
\label{thm:als-upper-recovery}
Universal positive choices of
$C_{\rm rank},C_{\rm burn},C_{\rm cert},C_{\rm stop},C_{\rm rep}$ make the
following true under Assumptions~\ref{assump:als-upper-base-scale}--%
\ref{assump:als-upper-confidence}.  There is a smoothing event $E_{\rm sm}$
of probability at least $1-\delta_{\rm sm}$ on which every realized column has norm at least
$(2\kappa_0)^{-1}$,
\[
\begin{aligned}
 q_{\rm real}&\le q_*, & \Gamma&\le1.01,\\
 \lambda_{\min}((V\odot W)^\top(V\odot W))&\ge1-q_*^2,
 &\lambda_{\min}((U\odot W)^\top(U\odot W))&\ge1-q_*^2,\\
 \lambda_{\min}((U\odot V)^\top(U\odot V))&\ge1-q_*^2.
\end{aligned}
\]
Conditional on any fixed instance in $E_{\rm sm}$, uniformly
in $i,j$,
\[
 \Pr_{\rm prop}(\mathcal E_{ij})
 =\Theta\!\left(r^{-5/3}(\log r)^{-3/2}\right),
\]
with universal comparison constants.  Hence the displayed $k=o(r^2)$ gives
one-run simultaneous target coverage with probability at least $26/27$.

On coverage, certification, unlabeled clustering, synchronized landing, and
cyclic refinement return at most $k$ terms with relative Frobenius residual at
most $\epsilon$.  With
\[
 J=\max\left\{1,\left\lceil C_{\rm rep}\log(1/\delta_{\rm init})\right\rceil\right\}
\]
independent completed runs, conditional success is at least
$1-\delta_{\rm init}$, so joint success is at least
$(1-\delta_{\rm sm})(1-\delta_{\rm init})$.  Every tape terminates, and the
dense arithmetic work is polynomial in
$n,r,\kappa_0,\rho^{-1},\log(1/\epsilon),\log(1/\delta_{\rm init})$.
\end{theorem}

\paragraph{Discussion.}
This is \textbf{partial progress}.  It preserves the smoothed CP
model, the strictly subquadratic rank
$k=\Theta(r^{5/3}(\log r)^{5/2})$, arbitrary relative accuracy, polynomial
time, and nested high-probability success.  Its partiality comes from added
bounded-scale, weak-interference, weight-balance, and smoothing/dimension
assumptions.  The chosen mechanism is a certified proposal,
synchronized-landing, and cyclic-ALS pipeline with independent restarts.

\noindent\textbf{Technical role and remaining barrier.}
The source setting leaves two technical barriers.  First, arbitrary base
geometry can make the smoothed components indistinguishable and prevent the
proposal estimates from remaining uniform across components.  Second,
independently proposed mode estimates do not automatically enter compatible
target spans.  Bounded scales and near-balanced weights control component
magnitudes, weak interference separates components, and the
smoothing/dimension margin preserves these properties after perturbation;
together they close the proposal recurrences.  Certification and synchronized
landing then align the three modes before cyclic ALS refinement.  The
remaining problem is to prove the same rank, accuracy, running-time, and
probability guarantees for arbitrary source-admissible base triples without
the added structural assumptions.

\subsubsection{Subproblem 2: Lower Bound}

\begin{openquestion}
Is there a universal $c>0$ such that, for $r<k\le r^{1+c}$, ALS, gradient
descent, or another iterative method converges with constant probability to a
strictly positive objective value?
\end{openquestion}

\paragraph{\underline{Perspective 1.}}

\paragraph{Formalized setting and preliminaries.}
\label{par:res-tensor-op3-p1i2-setting}
Let $r,n,k$ be positive integers, let $q>0$, and set $\rho=r^{-q}$.  For arbitrary deterministic
$\bar A,\bar B,\bar C\in\mathbb R^{n\times r}$, independently smooth their
columns by $\mathcal N(0,\rho^2I_n/n)$ to obtain
\[
 T=\sum_{j=1}^r a_j\otimes b_j\otimes c_j.
\]
For $X,Y,Z\in\mathbb R^{n\times k}$, set
$S(X,Y,Z)=\sum_i x_i\otimes y_i\otimes z_i$ and
$F(X,Y,Z)=\frac12\|T-S(X,Y,Z)\|_F^2$.
Here $T_{(m)}$ is the mode-$m$ matricization and $\odot$ is the Khatri--Rao
product.

For each $M\in\{\mathrm{cALS},\mathrm{cGD}\}$, draw an independent Gaussian
triple $G_x^M,G_y^M,G_z^M$ with iid $\mathcal N(0,1/n)$ entries.  Let
$Q_M=\operatorname{orth}(G_x^M)$, where $\operatorname{orth}$ is a fixed
measurable choice of orthonormal basis for the column space,
$\mathcal S_M=\operatorname{range}(G_x^M)$, and
\[
 \mathcal H_M=\mathcal S_M\otimes\mathbb R^n\otimes\mathbb R^n,
 \qquad P_{\mathcal H_M}=(Q_MQ_M^\top)\otimes I_n\otimes I_n.
\]
The methods share $T$ but have independent starts.  The \textbf{constrained sequential
ALS method} initializes at its Gaussian triple and, in $X,Y,Z$ order, uses
\[
\begin{aligned}
 X_{t+1}&=Q_{\rm cALS}Q_{\rm cALS}^\top T_{(1)}K_t^x
             ((K_t^x)^\top K_t^x)^\dagger,&K_t^x&=Z_t\odot Y_t,\\
 Y_{t+1}&=T_{(2)}K_t^y((K_t^y)^\top K_t^y)^\dagger,&K_t^y&=Z_t\odot X_{t+1},\\
 Z_{t+1}&=T_{(3)}K_t^z((K_t^z)^\top K_t^z)^\dagger,&K_t^z&=Y_{t+1}\odot X_{t+1}.
\end{aligned}
\]
Thus \textbf{only $X_t$ is constrained to its fixed initialization span}.
All three displayed updates are Moore--Penrose minimum-norm solves, including
when a design Gram is singular.

The \textbf{constrained GD method} writes $X_t=Q_{\rm cGD}C_t$ and
$f_Q(C,Y,Z)=F(Q_{\rm cGD}C,Y,Z)$.  Initialize at
\[
 (C_0,Y_0,Z_0)
 =(Q_{\rm cGD}^\top G_x^{\rm cGD},G_y^{\rm cGD},G_z^{\rm cGD}),
\]
and, testing $j=0,1,2,\ldots$ in increasing order, choose the first dyadic
$\eta_t=2^{-j}$ satisfying, for
$u_t=(C_t,Y_t,Z_t)$,
\[
 f_Q(u_t-\eta_t\nabla f_Q(u_t))
 \le f_Q(u_t)-\frac{\eta_t}{2}\|\nabla f_Q(u_t)\|_F^2,
\]
and take that step.  Both methods iterate the displayed updates from their
specified starts.  Write
$S_t^M=S(X_t^M,Y_t^M,Z_t^M)$ and $F_M(t)=F(X_t^M,Y_t^M,Z_t^M)$.
The gradient norm is the combined Frobenius norm of the $C,Y,Z$ blocks.

\paragraph{Technical assumptions.}
\begin{assumption}[Fixed ambient dimension]
\label{assump:res-tensor-op3-p1i2-dimension}
$n\ge8r^{5/4}$.
\end{assumption}

\begin{assumption}[Superlinear algorithmic rank]
\label{assump:res-tensor-op3-p1i2-rank-window}
$r<k\le r^{5/4}$; equivalently, the exponent is $\alpha=1/4$.
\end{assumption}

\begin{assumption}[Uniform arbitrary deterministic bases]
\label{assump:res-tensor-op3-p1i2-arbitrary-base}
The deterministic base triple is unrestricted, and the claim is pointwise
uniform over all such triples.
\end{assumption}

\begin{assumption}[Independent Gaussian smoothing]
\label{assump:res-tensor-op3-p1i2-gaussian-smoothing}
$q>0$ is fixed, $\rho=r^{-q}$, and the $3r$ smoothing vectors are mutually
independent with covariance $(\rho^2/n)I_n$.
\end{assumption}

\begin{assumption}[Shared target and independent starts]
\label{assump:res-tensor-op3-p1i2-joint-initialization}
The two initialization triples are mutually independent and independent of
smoothing, while both methods use the same realized $T$.
\end{assumption}

\paragraph{Main theory.}
\begin{theorem}[Fixed-span positive limiting objective]
\label{thm:res-tensor-op3-p1i2}
Under Assumptions~\ref{assump:res-tensor-op3-p1i2-dimension}--%
\ref{assump:res-tensor-op3-p1i2-joint-initialization}, with probability at
least $1/4$ over smoothing and both starts, simultaneously for
$M\in\{\mathrm{cALS},\mathrm{cGD}\}$,
\[
 \|(I-P_{\mathcal H_M})T\|_F^2\ge\frac34\|T\|_F^2,
 \qquad F_M(t)\ge\frac38\|T\|_F^2\quad\text{for every }t\ge0,
\]
and $F_M(t)$ has a finite scalar limit satisfying
\[
 \lim_{t\to\infty}F_M(t)\ge\frac38\|T\|_F^2.
\]
The statement holds for every admissible $r,n,k$ and is pointwise uniform in
the deterministic base triple.
\end{theorem}

\paragraph{Discussion.}
This is \textbf{partial progress}.  It preserves arbitrary base
factors, the smoothed tensor model, Gaussian initialization, every
$r<k\le r^{5/4}$, and the goal of a strictly positive limiting loss with
constant probability.  \emph{The main change is algorithmic: cALS and cGD constrain
one factor to its initialization span.}  The result also assumes
$n\ge8r^{5/4}$ and uses specified minimum-norm ALS and Armijo-GD updates.

\noindent \textbf{Technical role and remaining barrier.}
The source algorithms present a trajectory-control barrier: an unconstrained
update can leave the initialization span, so the proof loses the fixed
subspace used to obstruct exact recovery.  Constraining the $X$ factor keeps
every represented tensor in the fixed space
$\mathcal S_M\otimes\mathbb R^n\otimes\mathbb R^n$.  With constant
probability, the target has a nonzero component orthogonal to this space,
which forces positive loss.  The dimension condition gives $k/n\le1/8$ and
hence a constant-probability orthogonal-residual witness, while the specified
descent rules ensure that the loss converges.  The remaining problem is to
prove, under the same smoothed model, dimension condition, and rank window,
the corresponding constant-probability positive-limit result for
unconstrained sequential minimum-norm ALS and full-variable Armijo GD.
Returning to the full source contract additionally requires removing or
justifying the dimension restriction $n\ge8r^{5/4}$.

\paragraph{\underline{Perspective 2.}}

\paragraph{Formalized setting and preliminaries.}
Fix $\kappa\ge1$, $q>0$, and positive integers $n,r,k$, and put
$\rho=r^{-q}$.  For deterministic
$\bar A,\bar B,\bar C\in\mathbb R^{n\times r}$ with nonzero columns, let
$\widetilde A,\widetilde B,\widetilde C$ be their column-normalized versions.
Independently over $j$ and modes, draw
$\xi_j^a,\xi_j^b,\xi_j^c\sim\mathcal N(0,\rho^2I_n/n)$, set
$a_j=\bar a_j+\xi_j^a$, $b_j=\bar b_j+\xi_j^b$, and
$c_j=\bar c_j+\xi_j^c$, and form
$T=\sum_{j=1}^r a_j\otimes b_j\otimes c_j$.  For
$X,Y,Z\in\mathbb R^{n\times k}$, let
\[
 \widehat T(X,Y,Z)=\sum_{i=1}^k x_i\otimes y_i\otimes z_i,
 \qquad \mathcal L(X,Y,Z)=\|T-\widehat T(X,Y,Z)\|_F^2.
\]
Write $\widehat T_t=\widehat T(X_t,Y_t,Z_t)$ and use the standard mode
matricizations and Khatri--Rao product below.
From an iid $\mathcal N(0,1/n)$ initialization, compute the three
minimum-norm least-squares candidates in parallel from the old iterate:
\[
\begin{aligned}
 U_t^x&=Z_t\odot Y_t,&
 X_{t+1}^{\rm ls}&=T_{(1)}U_t^x((U_t^x)^\top U_t^x)^\dagger,\\
 U_t^y&=Z_t\odot X_t,&
 Y_{t+1}^{\rm ls}&=T_{(2)}U_t^y((U_t^y)^\top U_t^y)^\dagger,\\
 U_t^z&=Y_t\odot X_t,&
 Z_{t+1}^{\rm ls}&=T_{(3)}U_t^z((U_t^z)^\top U_t^z)^\dagger.
\end{aligned}
\]
Set $(X_{t+1}^{\rm raw},Y_{t+1}^{\rm raw},Z_{t+1}^{\rm raw})$ to the
componentwise averages of the old factors and these candidates.  For each
positive-norm component triple, replace all three norms by their geometric
mean; if a norm is zero, replace the triple by $(0,0,0)$.  This gauge preserves
the represented tensor.  The method is unconstrained half-relaxed parallel
ALS.

Let $\Lambda_A=(\bar A^\top\bar A)^{-1}\bar A^\top$, and define
$\Lambda_B,\Lambda_C$ analogously and
$Q=\Lambda_A\otimes\Lambda_B\otimes\Lambda_C$.  Put
\[
 p_{i,t}=(\Lambda_Ax_{i,t})\otimes(\Lambda_By_{i,t})
          \otimes(\Lambda_Cz_{i,t}),\quad
 D_r=\sum_{j=1}^r e_j^{\otimes3},\quad
 C_t=\sum_{i=1}^k p_{i,t},\quad
 \mathcal S_t=\operatorname{span}\{p_{i,t}:i\in[k]\},
\]
$P_t=\operatorname{Proj}_{\mathcal S_t}$,
$\Delta_0=\operatorname{dist}_F(D_r,\mathcal S_0)$, and
$E_\rho=QT-D_r$.  These definitions give the exact identity
\[
 Q(T-\widehat T_t)=D_r+E_\rho-C_t,\qquad C_t\in\mathcal S_t.
\]
The spaces $\mathcal S_t$ evolve adaptively, and both coefficient and
ambient tensor spaces use Frobenius geometry.
For $L_P<\delta/4$ and $\zeta<\delta/4$, define
$\mathsf C_2(\delta,L_P,\zeta,C_T)$ by four clauses:
\begin{enumerate}
\item $\Delta_0\ge\delta\|D_r\|_F=\delta\sqrt r$;
\item $\sum_{t\ge0}\|P_{t+1}-P_t\|_{\rm op}\le L_P$;
\item $\sum_{t\ge0}\|\widehat T_{t+1}-\widehat T_t\|_F<\infty$;
\item $\|E_\rho\|_F\le\zeta\|D_r\|_F$ and
      $\|T\|_F\le C_T\|D_r\|_F$, where $C_T=C_T(\kappa,q)$ is independent
      of $r,n,k$ and the base triple.
\end{enumerate}
\textbf{The event $\mathsf C_2$ is an outcome-dependent certificate.}
Its role is transparent: the projector-path clause preserves the initial
coefficient deficit at every time,
\[
 \operatorname{dist}_F(D_r,\mathcal S_t)
 \ge(\delta-L_P)\|D_r\|_F,
\]
while the same-target identity, the smoothing clause, and
$\|Q\|_{\rm op}\le\kappa^6$ give the all-time physical residual bound
\[
 \|T-\widehat T_t\|_F
 \ge\frac{\delta-L_P-\zeta}{\kappa^6C_T}\|T\|_F.
\]
The unsquared finite-variation clause makes $(\widehat T_t)_t$ Cauchy, so the
objective has a finite limit.

\paragraph{Technical assumptions.}
\begin{assumption}[Ambient dimension]\label{assump:dimension_2}
$n\ge C_{\rm dim}(\kappa,q)r^4\log r$.
\end{assumption}

\begin{assumption}[Full superlinear rank window]\label{assump:rank_window_2}
$r<k\le r^{5/4}$.
\end{assumption}

\begin{assumption}[Well-conditioned deterministic bases]\label{assump:base_conditioning_2}
Every base-column norm and every singular value of each column-normalized
base matrix lie in $[\kappa^{-1},\kappa]$.
\end{assumption}

\begin{assumption}[Independent Gaussian smoothing]\label{assump:gaussian_smoothing_2}
$\rho=r^{-q}$, and the $3r$ smoothing vectors have the independent Gaussian
law above.
\end{assumption}

\begin{assumption}[Independent Gaussian initialization]\label{assump:independent_initialization_2}
All $3nk$ initial factor entries are iid $\mathcal N(0,1/n)$ and independent
of smoothing.
\end{assumption}

All probabilities refer to the joint law conditional on the deterministic
base triple.

\paragraph{Main theory.}
\begin{theorem}[Conditional positive limiting loss]
Under Assumptions~\ref{assump:dimension_2}--%
\ref{assump:independent_initialization_2},
let $r_0\in\mathbb N$ and $C_{\rm dim},\delta,L_P,\zeta,C_T>0$ depend only
on $(\kappa,q)$, with $L_P<\delta/4$ and $\zeta<\delta/4$, and define
\[
 \epsilon=\left(\frac{\delta-L_P-\zeta}{\kappa^6C_T}\right)^2>0.
\]
For every $r\ge r_0$, admissible $n,k$ and base triple, the displayed
 half-relaxed trajectory satisfies the deterministic event inclusion
\[
 \mathsf C_2(\delta,L_P,\zeta,C_T)\subseteq
 \left\{\begin{array}{l}
 \lim_{t\to\infty}\mathcal L(X_t,Y_t,Z_t)\text{ exists and is finite},\\[1mm]
 \lim_{t\to\infty}\mathcal L(X_t,Y_t,Z_t)\ge\epsilon\|T\|_F^2
 \end{array}\right\}.
\]
The inclusion is pointwise under the joint smoothing-and-initialization law
conditional on the base triple.  In the
separate deterministic zero-smoothing
baseline with $n\ge r$ and orthonormal base columns,
$QT=D_r$, $E_\rho=0$, $\|Q\|_{\rm op}=1$, and
$\|T\|_F=\|D_r\|_F$; the first three applicable certificate clauses imply
\[
 \lim_{t\to\infty}\mathcal L(X_t,Y_t,Z_t)
 \ge(\delta-L_P)^2\|T\|_F^2.
\]
\end{theorem}

\paragraph{Discussion.}
This is \textbf{partial progress}.  It preserves the smoothed CP
loss, Gaussian initialization, every $r<k\le r^{5/4}$, and the target of a
positive relative limiting loss.  The decisive change is that \emph{the conclusion
is conditioned on the trajectory certificate}
$\mathsf C_2(\delta,L_P,\zeta,C_T)$, which requires an initial deficit,
controlled motion of the adaptive span, controlled smoothing and scale, and
finite tensor variation.  The theorem also uses well-conditioned bases, a
high-dimensional regime, and half-relaxed balanced parallel ALS.

\noindent\textbf{Technical role and remaining barrier.}
The source setting leaves two trajectory barriers: the adaptive coefficient
span can rotate until it absorbs the initial deficit, and the represented
tensor is not known to converge.  The certificate controls projector motion
to preserve the deficit, uses the smoothing and scale bounds to transfer that
deficit to a physical residual, and imposes finite tensor variation to make
the trajectory Cauchy.  Well-conditioning controls the coefficient-to-tensor
transfer, while the dimension regime and modified dynamics support the
certificate conditions.  These clauses yield a positive limiting loss.  The
immediate missing proposition is to prove that, for suitable constants
depending only on $(\kappa,q)$,
\[
 \Pr[\mathsf C_2(\delta,L_P,\zeta,C_T)]
 \ge p_0(\kappa,q)>0
\]
uniformly over the admitted parameters and bases.  Even this would establish
only the restricted well-conditioned, high-dimensional, half-relaxed balanced
parallel-ALS result.  Returning to the source problem additionally requires a
constant-probability positive-limit theorem without those geometric,
dimensional, and algorithmic restrictions.

% \newpage
\paragraph{\underline{Perspective 3.}}

\paragraph{Formalized setting and preliminaries.}
Fix $\kappa\ge1$, $q\ge4$, and positive integers $n,r,k$, and put
$\rho=r^{-q}$.  For deterministic
$\bar A,\bar B,\bar C\in\mathbb R^{n\times r}$ with nonzero columns, write
$\bar A^\circ,\bar B^\circ,\bar C^\circ$ for their column-normalized versions.
Independently over
$j$ and modes, draw
$\xi_j^a,\xi_j^b,\xi_j^c\sim\mathcal N(0,\rho^2I_n/n)$ and set
$a_j=\bar a_j+\xi_j^a$, $b_j=\bar b_j+\xi_j^b$, and
$c_j=\bar c_j+\xi_j^c$.  With $A=[a_j]$, $B=[b_j]$, and $C=[c_j]$, write
\[
 T=\sum_{j=1}^r a_j\otimes b_j\otimes c_j=(A\otimes B\otimes C)D_r,\qquad
 D_r=\sum_{j=1}^r e_j^{\otimes3},\qquad
 F(X,Y,Z)=\left\|T-\sum_{i=1}^k x_i\otimes y_i\otimes z_i\right\|_F^2.
\]
Let $\mathcal G$ balance every component with three positive norms to their
geometric mean, preserving its rank-one product, and leave a component with a
zero factor unchanged.  Apply $\mathcal G$ to an iid
$\mathcal N(0,1/n)$ initialization and after every simultaneous full-batch
gradient step, where $\theta_t=(X_t,Y_t,Z_t)$:
\[
 \widetilde X_{t+1}=X_t-\eta\nabla_XF(\theta_t),\quad
 \widetilde Y_{t+1}=Y_t-\eta\nabla_YF(\theta_t),\quad
 \widetilde Z_{t+1}=Z_t-\eta\nabla_ZF(\theta_t),\qquad
 \eta=(nkr)^{-12}.
\]
This defines balanced full-variable gradient descent.

On the full-rank event, set
$\alpha_{i,t}=A^\dagger x_{i,t}$ and define $\beta_{i,t},\gamma_{i,t}$
analogously, and set
$\bar\alpha_{i,0}=\sqrt{n/r}\,\alpha_{i,0}$, with the same convention for
$\bar\beta_{i,0},\bar\gamma_{i,0}$.  Put
\[
 \widehat D_0=\sum_i\alpha_{i,0}\otimes\beta_{i,0}\otimes\gamma_{i,0},
 \qquad \delta_0=\frac18,
\]
and let $\mathscr S_0$ be the span, over $i\in[k]$, of
\[
 \mathbb R^r\otimes\beta_{i,0}\otimes\gamma_{i,0},\quad
 \alpha_{i,0}\otimes\mathbb R^r\otimes\gamma_{i,0},\quad
 \alpha_{i,0}\otimes\beta_{i,0}\otimes\mathbb R^r.
\]
Replacing the fixed coordinates by their barred versions leaves this tangent
span unchanged.  With $\kappa_1=2\kappa^2$, define
$\mathcal E_{\rm init\_norm}$ as the intersection of:
\begin{enumerate}
\item $\|M\|_{\rm op}\le\kappa_1$ and
      $\sigma_{\min}(M)\ge\kappa_1^{-1}$ for $M=A,B,C$;
\item all eigenvalues of the three normalized pair Grams
      $[\bar\beta_{i,0}\otimes\bar\gamma_{i,0}]_i^\top
       [\bar\beta_{i,0}\otimes\bar\gamma_{i,0}]_i$ and its cyclic analogues lie in
      $[r^{-20},r^{20}]$;
\item some unit $W_0\perp\mathscr S_0$ satisfies
      $\langle D_r-\widehat D_0,W_0\rangle_F
       \ge\delta_0\|D_r\|_F$;
\item $\max_{i,m\in\{x,y,z\}}\|m_{i,0}\|_2\le2$.
\end{enumerate}
The normalized pair Grams equal $(n/r)^2$ times their raw counterparts.  For
$\theta=(X,Y,Z)$, define
\[
 d_{\rm bal}(\theta,\theta')^2
 =\|X-X'\|_F^2+\|Y-Y'\|_F^2+\|Z-Z'\|_F^2,
 \quad E_{\rm path}=\sum_{t\ge0}d_{\rm bal}(\theta_{t+1},\theta_t),
\]
and $C_{\rm CP}(\kappa,R)=\kappa_1^3(1+3R)$.  The \textbf{sole trajectory
certificate} is
\[
 E_\star=\min\left\{1,
 \sqrt{\frac{\delta_0}{16C_{\rm CP}(\kappa,3)}}\right\},
 \qquad \mathcal C_{\rm path}=\{E_{\rm path}\le E_\star\}.
\]
\paragraph{Technical assumptions.}
\begin{assumption}[Well-conditioned deterministic bases]
\label{assump:base_conditioning}
Every base-column norm and every singular value of each column-normalized
base matrix lie in $[\kappa^{-1},\kappa]$.
\end{assumption}

\begin{assumption}[Smoothed dimension regime]
\label{assump:dimension}
$q\ge4$ is fixed, $r$ is sufficiently large, and
$n\ge C(\kappa,q)r^4\log r$.
\end{assumption}

\begin{assumption}[Universal superlinear rank window]
\label{assump:rank_window}
$r<k\le\lfloor r^{5/4}\rfloor$.
\end{assumption}

\begin{assumption}[Independent Gaussian smoothing]
\label{assump:gaussian_smoothing}
All smoothing vectors have the independent law
$\mathcal N(0,r^{-2q}I_n/n)$ and are independent of initialization.
\end{assumption}

\begin{assumption}[Gaussian initialization]
\label{assump:independent_initialization}
Before balancing, all entries of $X_0^{\rm raw},Y_0^{\rm raw},Z_0^{\rm raw}$
are iid $\mathcal N(0,1/n)$.
\end{assumption}

\begin{assumption}[Fixed balanced gradient-descent protocol]
\label{assump:gd_step}
The protocol uses the simultaneous full-batch update, step size
$\eta=(nkr)^{-12}$, and map $\mathcal G$ specified above.
\end{assumption}

\paragraph{Main theory.}
\begin{theorem}[Conditional positive-loss certificate]
\label{thm:main}
Under Assumptions~\ref{assump:base_conditioning}--\ref{assump:gd_step}, there
are $r_0(\kappa,q)$ and $C(\kappa,q)$ such that, uniformly over every
\[
 r\ge r_0(\kappa,q),\qquad n\ge C(\kappa,q)r^4\log r,
 \qquad r<k\le\lfloor r^{5/4}\rfloor,
\]
and every admissible deterministic base triple,
\[
 \Pr(\mathcal E_{\rm init\_norm})\ge1-r^{-10}.
\]
With
\[
 \epsilon_0(\kappa)=\left(\frac{15}{16}\delta_0\right)^2
 \kappa_1^{-12}>0,
\]
on $\mathcal E_{\rm init\_norm}\cap\mathcal C_{\rm path}$ the balanced
iterates converge in $d_{\rm bal}$ to a finite $\theta_\infty$ and
\[
 \lim_{t\to\infty}F(\theta_t)=F(\theta_\infty)
 \ge\epsilon_0(\kappa)\|T\|_F^2>0.
\]
Consequently, if $\mathcal F_+$ is this convergence-and-positive-limit event,
\[
 \Pr(\mathcal F_+)\ge(1-r^{-10})
 \Pr(\mathcal C_{\rm path}\mid\mathcal E_{\rm init\_norm}).
\]
Probabilities are under the joint smoothing-and-initialization law conditional
on the deterministic base triple.
\end{theorem}

\paragraph{Discussion.}
This is \textbf{partial progress}.  It preserves the smoothed CP
loss, Gaussian initialization, every
$r<k\le\lfloor r^{5/4}\rfloor$, simultaneous all-factor gradient steps, and
the target of a positive relative limiting loss.  \emph{The decisive added
condition is the finite-path event} $\mathcal C_{\rm path}$.  The analyzed
dynamics additionally use well-conditioned bases, a high-dimensional regime,
a tiny fixed step, and product-preserving balancing after every step.

\noindent\textbf{Technical role and remaining barrier.}
The source setting leaves a global trajectory barrier: a tiny step controls
each update locally but does not bound the total motion, accumulated nonlinear
error, or convergence of the iterates.  The finite-path event bounds this
total motion, keeps the balanced iterates in the region where the initial
tangent-deficit witness survives, and controls the accumulated Taylor
remainder.  Well-conditioning transfers this witness to physical loss, while
product-preserving balancing controls factor scales along the finite path.
This yields convergence to a positive limiting loss.  The
immediate missing proposition is to prove that, for some
$p_0(\kappa,q)>0$,
\[
 \Pr(\mathcal C_{\rm path}\mid\mathcal E_{\rm init\_norm})
 \ge p_0(\kappa,q)
\]
uniformly over the admitted parameters and bases.  Even this would establish
only the restricted well-conditioned, high-dimensional, tiny-step balanced
method.  Returning to the source problem additionally requires removing these
geometric, dimensional, step-size, and balancing restrictions.

\subsection{Is Interaction Necessary for Order-Optimal 1-bit Mean Estimation?}
\label{sec:case-one-bit}

\begin{openquestion}
Consider one-dimensional mean estimation over the nonparametric distribution family
\[
  \mathcal D(k,\lambda,\sigma)
  =\left\{D:\mu(D):=\mathbb E_{X\sim D}[X]\in[-\lambda,\lambda],\ 
  \mathbb   E_{X\sim D}|X-\mu(D)|^k\le\sigma^k\right\},
\]
where $k > 1$ and $\lambda \geq \sigma > 0$ are known to the learner. A 1-bit communication protocol observes independent samples $X_1, ... , X_n \sim D$ only through binary messages $Y_t = \boldsymbol{1}\{X_t \in A_t\}$, where $A_t \subset \mathbb R$ is measurable.  In a fully non-adaptive protocol, all sets $A_1, ... , A_n$ are fixed before any messages are observed, possibly using public or private randomness. Threshold and interval queries correspond to $A_t$ being a half-line or an interval, respectively. We say that a protocol is $(\epsilon, \delta)$-accurate over $\mathcal{D}(k, \lambda, \sigma)$ if its output $\hat{\mu}$ satisfies
\[
\sup_{D\in \mathcal{D}(k, \lambda, \sigma)} \mathbb P \{|\hat{\mu}-\mu(D)| > \epsilon\} \leq \delta.
\]
The adaptive 1-bit minimax sample complexity is known:
\[
  r_k(\lambda,\sigma,\epsilon,\delta)
  =\log\frac{\lambda}{\sigma}+
  \begin{cases}
    \dfrac{\sigma^2}{\epsilon^2}\log\dfrac1\delta,&k>2,\\[0.45em]
    \dfrac{\sigma^2}{\epsilon^2}\log\dfrac{\sigma}{\epsilon}
    \log\dfrac1\delta,&k=2,\\[0.45em]
    \left(\dfrac{\sigma}{\epsilon}\right)^{k/(k-1)}
    \log\dfrac1\delta,&1<k<2.
  \end{cases}
\]
The open question of \citet{lau2026interaction} asks \textit{whether fully non-adaptive arbitrary 1-bit quantizers can achieve the adaptive minimax rate?}
\end{openquestion}

\subsubsection{Subproblem 1: Order-Optimal Non-Adaptive 1-Bit Mean Estimation}
\begin{openquestion}
Fix $k>1$. \textit{Do there exist constants $c_k, C_k > 0$ such that, for all $\lambda \geq \sigma > 0$, all $0 < \epsilon\leq c_k\sigma$, and all $\delta \in (0, 1/2)$,
there is a fully non-adaptive 1-bit protocol that is $(\epsilon, \delta)$-accurate over $\mathcal{D}(k, \lambda, \sigma)$ using at most $n \leq C_k r_k(\lambda, \sigma, \epsilon, \delta)$ samples?}
\end{openquestion}

\paragraph{\underline{Perspective 1.}}

\paragraph{Formalized setting and preliminaries.}
Fix \(k>1\), known \(\lambda\geq\sigma>0\), accuracy \(\epsilon>0\), and
confidence \(\delta\in(0,1/2)\).  For a law \(D\) on \(\mathbb R\), define
\[
\mathcal D(k,\lambda,\sigma)
:=\left\{D:\ \mu(D)=\mathbb E_DX\in[-\lambda,\lambda],\quad
\mathbb E_D|X-\mu(D)|^k\leq\sigma^k\right\}.
\]
This is the unrestricted central-$k$-moment class.  The target
sample complexity is
\[
r_k(\lambda,\sigma,\epsilon,\delta)
:=\log\frac{\lambda}{\sigma}+
\begin{cases}
\dfrac{\sigma^2}{\epsilon^2}\log\dfrac1\delta,&k>2,\\[0.45em]
\dfrac{\sigma^2}{\epsilon^2}\log\dfrac{\sigma}{\epsilon}
\log\dfrac1\delta,&k=2,\\[0.45em]
\left(\dfrac{\sigma}{\epsilon}\right)^{k/(k-1)}
\log\dfrac1\delta,&1<k<2.
\end{cases}
\]

Split the indices in advance into localization and refinement blocks.
Localization uses the fully non-adaptive balanced-code construction of
\citet[Theorem~16]{lau2026sequential} at confidence \(\delta/4\): it uses
precommitted Borel union-of-cell queries and returns a center \(c\) defined for
every transcript.  Write \(N_{\rm loc}\) for the number of localization
queries.  For finite constants \(L_k,C_{{\rm loc},k}\) depending only on
\(k\), its theorem-level contract is
\[
 \Pr\{|c-\mu(D)|\leq L_k\sigma\}\geq1-\frac{\delta}{4},\qquad
 N_{\rm loc}\leq C_{{\rm loc},k}\left(1+\log\frac{\lambda}{\sigma}
 +\log\frac4\delta\right).
\]
The \textbf{refinement bank is also fixed before communication}.  For \(k\)-only
constants \(\gamma_k\in(0,1)\) and \(b_k\geq1\), set
\[
h_0=\gamma_k\epsilon,\qquad
H_*=b_k\sigma\left(\frac{\sigma}{\epsilon}\right)^{1/(k-1)},\qquad
J=\left\lceil\log_2\frac{H_*}{h_0}\right\rceil,
\qquad h_j=2^jh_0.
\]
For \(0\leq j<J\), give the fine levels \(h_j\leq\sigma\) weights
\(h_j/\sigma\) and the coarse levels \(h_j>\sigma\) weights
\((h_j/\sigma)^{2-k}\); normalize within each nonempty group and give each
nonempty group equal total mass.  Denote the resulting level probabilities by
\(p_j\).  With \(\mathcal S=\{0,1/4,1/2,3/4\}\), define
\[
Q_{j,a}(x)=ah_j+h_j\left\lfloor\frac{x-ah_j}{h_j}\right\rfloor,
\qquad F_{j,a,b}=Q_{j,a}-Q_{j+1,b}.
\]
Independently for every refinement sample \(i\), draw
\(L_i\sim(p_0,\ldots,p_{J-1})\),
\(A_i,B_i\stackrel{\rm iid}{\sim}{\rm Unif}(\mathcal S)\), and
\(U_i\sim{\rm Unif}[-1,2]\), all before any response, and transmit
\[
Y_i=\mathbf1\!\left\{
\frac{F_{L_i,A_i,B_i}(X_i)}{h_{L_i}}\geq U_i\right\}.
\]

At decoding time, choose the unique \(a_j(c)\in\mathcal S\) such that
\(\{c/h_j-a_j(c)\}\in[3/8,5/8)\), and form
\[
\begin{split}
Z_i(c)&=\frac{48h_{L_i}}{p_{L_i}}
\mathbf1\{(A_i,B_i)=(a_{L_i}(c),a_{L_i+1}(c))\}\\[-0.2em]
&\quad\times\left[Y_i-\mathbf1\!\left\{
\frac{F_{L_i,A_i,B_i}(c)}{h_{L_i}}\geq U_i\right\}\right].
\end{split}
\]
Choose a positive integer \(s\) and an odd positive integer \(q\), both as
functions of the public parameter tuple, partition the refinement indices in
advance into \(q\) blocks of size \(s\), and output
\[
\widehat\mu=c+\operatorname{median}_{1\leq g\leq q}
\left(\frac1s\sum_{i\in G_g}Z_i(c)\right).
\]
Localization supplies the decoder-side shift, centering, and importance
weights, while all queries remain precommitted.

\paragraph{Technical assumptions.}
\begin{assumption}[Parameter domain]\label{assump:parameter-domain}
The exponent \(k>1\) is fixed and known, \(\lambda\geq\sigma>0\) are known,
\(\delta\in(0,1/2)\), and \(0<\epsilon\leq c_k\sigma\), where
\(c_k\in(0,1)\) depends only on \(k\).
\end{assumption}

\begin{assumption}[Unrestricted central moment class]\label{assump:moment-class}
The samples have a common law \(D\in\mathcal D(k,\lambda,\sigma)\).
\end{assumption}

\begin{assumption}[Independent samples and precommitted seeds]
\label{assump:iid-independent-randomness}
All samples are independent with common law \(D\).  The sample split, median
blocks, localization randomness, and every refinement seed are mutually
independent where appropriate and fixed before the first response bit.
\end{assumption}

\paragraph{Main theory.}
\begin{theorem}[Order-optimal noninteractive one-bit mean estimation]
\label{thm:technical}\label{cor:rate}
Under Assumptions~\ref{assump:parameter-domain}--
\ref{assump:iid-independent-randomness}, one may take \(c_k=e^{-1}\),
\(\gamma_k=1/8\), choose \(b_k\) as a function only of \(k\), and choose
\(s\) and the odd \(q\) as functions of the public tuple
\((k,\lambda,\sigma,\epsilon,\delta)\).  The resulting Borel query bank is
fixed before the first response, uses exactly one bit from each of
\(n=N_{\rm loc}+qs\) independent samples, and satisfies, for a finite
\(C_k\) depending only on \(k\),
\[
n\leq C_k r_k(\lambda,\sigma,\epsilon,\delta),\qquad
\sup_{D\in\mathcal D(k,\lambda,\sigma)}
\Pr_{D,\,\mathrm{protocol}}
\{|\widehat\mu-\mu(D)|>\epsilon\}\leq\delta.
\]
The probability is unconditional over both sample blocks and all protocol
randomness, the horizon is deterministic, and the loss is absolute error on
\(\mathbb R\).
\end{theorem}

\paragraph{Discussion.}
This theorem matches the source problem's \textbf{full scope} over the
unrestricted central-\(k\)-moment class.  Together with the known one-bit
minimax lower bound, its rate is order-optimal.  Both query banks are
precommitted, and localization is a decoder-side operation.

\paragraph{\underline{Perspective 3.}}

\paragraph{Formalized setting and preliminaries.}
For \(k>1\) and known \(\lambda\geq\sigma>0\), define
\[
\mathcal D(k,\lambda,\sigma)
=\{D:\mu(D)=\mathbb E_DX\in[-\lambda,\lambda],\
\mathbb E_D|X-\mu(D)|^k\leq\sigma^k\}.
\]
This is the unrestricted central-$k$-moment class.
Its three-regime target rate is
\[
r_k(\lambda,\sigma,\epsilon,\delta)
=\log\frac{\lambda}{\sigma}+
\begin{cases}
\dfrac{\sigma^2}{\epsilon^2}\log\dfrac1\delta,&k>2,\\[0.4em]
\dfrac{\sigma^2}{\epsilon^2}\log\dfrac{\sigma}{\epsilon}
\log\dfrac1\delta,&k=2,\\[0.4em]
\left(\dfrac{\sigma}{\epsilon}\right)^{k/(k-1)}
\log\dfrac1\delta,&1<k<2.
\end{cases}
\]
Split the samples in advance into localization and refinement blocks.  The
first uses the coding-based fully non-adaptive localizer of
\citet[Theorem~16]{lau2026sequential} at confidence \(\delta/4\).  Let
\(R_{\rm loc}\) be a public localization seed, independent of all samples and
fixed before any response bit.  The localizer uses precommitted Borel queries
\(\mathcal B_i(R_{\rm loc})\) and an always-defined decoder
output
\[
c=\mathsf{Dec}_{\rm loc}
\bigl(R_{\rm loc},
(\mathbf1\{X_i\in\mathcal B_i(R_{\rm loc})\})_{i\in I_{\rm loc}}\bigr).
\]
For finite constants \(L_k,C_{{\rm loc},k}\) depending only on \(k\), the
localizer satisfies
\[
 \Pr\{|c-\mu(D)|\leq L_k\sigma\}\geq1-\frac{\delta}{4},\qquad
 N_{\rm loc}:=|I_{\rm loc}|\leq C_{{\rm loc},k}
 \left(1+\log\frac{\lambda}{\sigma}+\log\frac4\delta\right).
\]

For \(k\)-only \(a_k,b_k>0\), let
\[
\begin{aligned}
h_0&=a_k\sigma,
&H_*&=b_k\sigma(\sigma/\epsilon)^{1/(k-1)},
&S&=\left\lceil\log_2(H_*/h_0)\right\rceil,\\
h_s&=2^sh_0,
&p_s&=\frac{h_s^{2-k}}{Z_S},
&Z_S&=\sum_{s=0}^S h_s^{2-k}.
\end{aligned}
\]
On the dyadic grid, set
\(P_{s,j}=[jh_s,(j+1)h_s)\),
\(m_{s,j}=(j+1/2)h_s\), and
\(J_{s,j}=P_{s,j-1}\cup P_{s,j}\cup P_{s,j+1}\).  Define the rings
\[
\mathcal R_{0,j}=J_{0,j},\qquad
\mathcal R_{s,j,b}=J_{s,j}\setminus J_{s-1,2j+b}\quad(s\geq1),
\]
and color indices by
\(\mathcal J_{s,\ell}=\{j:j\equiv\ell\pmod4\}\).  Half-open cells fix all
boundary ties.

For each refinement sample \(i\), \textbf{independently precommit}
\(L_i\sim(p_0,\ldots,p_S)\),
\(C_i\sim{\rm Unif}\{0,1,2,3\}\),
\(U_i\sim{\rm Unif}[-1,1]\), and a countable independent Rademacher mask
\((\rho_{i,s,j})\).  If \(L_i=s\geq1\), also precommit independent
\(T_i\sim{\rm Unif}\{\mathrm{coord},\mathrm{mass}\}\) and
\(B_i\sim{\rm Unif}\{0,1\}\); at \(s=0\), take
\(T_i=\mathrm{coord}\).  With
\(\psi_{s,j,\mathrm{coord}}(x)=(x-m_{s,j})/(2h_s)\) and
\(\psi_{s,j,\mathrm{mass}}(x)=1\), define
\[
F_i(x)=
\begin{cases}
\displaystyle\sum_{j\in\mathcal J_{0,C_i}}\rho_{i,0,j}
\frac{x-m_{0,j}}{2h_0}\mathbf1_{\mathcal R_{0,j}}(x),&L_i=0,\\[0.9em]
\displaystyle\sum_{j\in\mathcal J_{s,C_i}}\rho_{i,s,j}
\psi_{s,j,T_i}(x)\mathbf1_{\mathcal R_{s,j,B_i}}(x),&L_i=s\geq1.
\end{cases}
\]
Each refinement sample transmits the single bit
\(Y_i=\mathbf1\{F_i(X_i)\geq U_i\}\).  The decoder computes
\(Y_i^0=\mathbf1\{0\geq U_i\}\) and
\(\Delta Y_i=Y_i-Y_i^0\).

After the complete transcript arrives, choose
\[
j_0(c)=\min\operatorname*{argmin}_{j\in\mathbb Z}|c-m_{0,j}|,
\qquad m_0(c):=m_{0,j_0(c)},
\]
and set for \(s\geq1\)
\[
j_s=\left\lfloor j_0(c)/2^s\right\rfloor,\qquad
b_s=j_{s-1}-2j_s,\qquad m_s=m_{s,j_s},\qquad d_s=m_s-m_0(c),
\]
with \(\kappa_s=j_s\bmod4\).  The decoder retains and reweights the
precommitted bits through
\[
W_i(c)=
\begin{cases}
\displaystyle\frac{16h_0}{p_0}\mathbf1\{C_i=\kappa_0\}
\rho_{i,0,j_0(c)}\Delta Y_i,&L_i=0,\\[0.9em]
\displaystyle\frac{16}{p_s}\mathbf1\{C_i=\kappa_s,B_i=b_s\}
\rho_{i,s,j_s}
\bigl[4h_s\mathbf1\{T_i=\mathrm{coord}\}
+2d_s\mathbf1\{T_i=\mathrm{mass}\}\bigr]\Delta Y_i,&L_i=s\geq1.
\end{cases}
\]
Finally, for \(k\)-only \(\alpha_k,\beta_k>0\), preassign
\[
G_\delta=2\left\lceil\alpha_k\log\frac8\delta\right\rceil+1,\qquad
B_{\rm ref}=\left\lceil\beta_k\frac{\sigma^kZ_S}{\epsilon^2}\right\rceil,
\]
equal-size median blocks and output
\[
\widehat\mu=m_0(c)+\operatorname{median}_{1\leq g\leq G_\delta}
\left(\frac1{B_{\rm ref}}\sum_{i\in G_g}W_i(c)\right).
\]
Localization selects the reconstruction path in the decoder, while all
queries remain precommitted.

\paragraph{Technical assumptions.}
\begin{assumption}[Parameter domain]\label{assump:parameter-domain_3}
\(k>1\) is fixed and known, \(\lambda\geq\sigma>0\) are known,
\(\delta\in(0,1/2)\), and \(0<\epsilon\leq c_k\sigma\), for a positive
\(k\)-only constant \(c_k<1\).
\end{assumption}

\begin{assumption}[Unrestricted central-moment class]\label{assump:moment-class_3}
The common law belongs to the unrestricted class
\(\mathcal D(k,\lambda,\sigma)\).
\end{assumption}

\begin{assumption}[Independent observations and seeds]
\label{assump:independent-samples_3}
Both sample blocks are i.i.d. from \(D\).  All localization and refinement
seeds are mutually independent where appropriate and independent of the
observations.
\end{assumption}

\begin{assumption}[Precommitted protocol]\label{assump:precommitted-protocol_3}
The split, all seeds and masks, and the median blocks are fixed before any bit
is observed.  Every query is Borel and independent of earlier messages.  The
localization output enters the decoder after collection of the full
transcript.
\end{assumption}

\paragraph{Main theory.}
\begin{theorem}[Order-optimal fully non-adaptive one-bit mean estimation]
\label{thm:main_3}
Under Assumptions~\ref{assump:parameter-domain_3}--
\ref{assump:precommitted-protocol_3}, for every fixed \(k>1\) there are
\(k\)-only constants \(c_k,C_k,a_k,b_k,\alpha_k,\beta_k>0\) such that the
protocol above is well defined and, for all admissible public parameters,
uses one bit per independent sample at the deterministic horizon
\(n=N_{\rm loc}+G_\delta B_{\rm ref}\) with
\[
n\leq C_k r_k(\lambda,\sigma,\epsilon,\delta),\qquad
\sup_{D\in\mathcal D(k,\lambda,\sigma)}
\Pr_{D,\,\mathrm{protocol}}
\{|\widehat\mu-\mu(D)|>\epsilon\}\leq\delta.
\]
The guarantee is unconditional over all observations and protocol randomness,
and uses absolute error on \(\mathbb R\).  It retains the exact three regimes
displayed above, including the single \(\log(\sigma/\epsilon)\) factor at
\(k=2\), with all constants depending only on \(k\).
\end{theorem}

\paragraph{Discussion.}
This theorem matches the source problem's \textbf{full scope} over the
unrestricted moment class.  Together with the known one-bit minimax lower
bound, its rate is order-optimal.  The masked multiscale bank is fully
precommitted, and localization selects the reconstruction path in the decoder.

\subsection[Is the Power of Deep Learning over Linear Models Inherently Distribution Dependent?]
{Is the Power of Deep Learning over Linear Models Inherently\\Distribution Dependent?}
\begin{openquestion}
\citet{ben2002limitations} defined the dimension complexity $\operatorname{dc}(\mathcal H)$ of a binary hypothesis class $\mathcal H \subseteq \{\pm 1\}^\mathcal X$  as the smallest dimension $d$ s.t. there exists a feature map $\varphi : \mathcal X \to \mathbb R^d$ allowing linear representation of $\mathcal H$ (i.e. s.t. $\forall_{h \in \mathcal H} \exists_{w \in \mathbb R^d} \forall_x h(x) = \operatorname{sign}(\langle w, \varphi(x)\rangle)$). Using dimension complexity ($\operatorname{dc}(\mathcal H)$) to measure the smallest shared feature dimension in which a hypothesis class is linearly realizable, the open question of \citet{feldman2026deep} asks \textit{whether distribution-independent SQ learning implies low dimension complexity, and whether anything learnable with (S)GD on a (benign) neural network under any input distribution is also learnable with a linear model.}
\end{openquestion}

\subsubsection{Subproblem 1: Learning with SGD over Neural Networks}
\begin{openquestion}
\textit{Is there a constant $C$ such that for all $\mathcal H \subseteq \{\pm 1\}^\mathcal X$ over $\mathcal{X} = \{\pm 1\}^n$, and $\epsilon < 1/4$, if there exists a fully connected ReLU network with $S$ parameters in total, stepsize $\eta$ and number of steps $T$, such that for every input distribution $\mathcal D$, every $h^* \in \mathcal H$, SGD yields expected error $\mathbb E \mathcal L_{\mathcal D,h^*} (\hat h) \leq \epsilon$ (expectation over the initialization and SGD sampling), then $\operatorname{dc}(\mathcal H) \leq C \cdot T S$.}
\end{openquestion}

\paragraph{\underline{Perspective 3.}}

\paragraph{Formalized setting and preliminaries.}
Fix \(n,m,T\geq1\), \(\eta>0\), and \(\varepsilon\geq0\).  Let
\(\mathcal X=\{-1,+1\}^n\),
\(\mathcal H\subseteq\{-1,+1\}^{\mathcal X}\), and fix a tie label
\(s_0\in\{-1,+1\}\).  Write
\(\operatorname{sign}_{s_0}(z)=\operatorname{sign}(z)\) for \(z\ne0\) and
\(\operatorname{sign}_{s_0}(0)=s_0\), and define the strict error of a score
\(g\) by
\[
\mathcal L_{\mathcal D,h}(g)
=\Pr_{x\sim\mathcal D}
\{\operatorname{sign}_{s_0}(g(x))h(x)<0\}.
\]
\textbf{Consider the coordinatewise ReLU activation
\(\sigma(z)=\max\{0,z\}\) and the bias-free depth-two ReLU network}
\[
f_{a,W}(x)=a^\top\sigma(Wx),\qquad
W\in\mathbb R^{m\times n},\quad a\in\mathbb R^m,\qquad
S=m(n+1),
\]
with both layers trainable.  Initialize independently with
\(W_{ji}^{(0)}\sim\mathcal N(0,1/n)\) and
\(a_j^{(0)}\sim\mathcal N(0,1/m)\).  Given fresh
\(x^{(t)}\sim\mathcal D\), run \(T\geq1\) one-sample, all-layer SGD steps
with fixed stepsize \(\eta>0\) and logistic loss
\(\ell(z)=\log(1+e^{-z})\).  Fix \(\kappa_{\rm kink}\in[0,1]\) once and use
the ReLU derivative
\[
 \sigma'_{\kappa_{\rm kink}}(z)
 =\mathbf1\{z>0\}+\kappa_{\rm kink}\mathbf1\{z=0\}
\]
in every gradient, so the SGD recursion is defined also at zero
preactivations:
\[
(a^{(t+1)},W^{(t+1)})=(a^{(t)},W^{(t)})
-\eta\nabla_{(a,W)}
\ell\!\left(h(x^{(t)})f_{a^{(t)},W^{(t)}}(x^{(t)})\right).
\]
The returned score is
\[
G_\omega(x)=\sum_{t=\lceil T/2\rceil}^{T}f_{a^{(t)},W^{(t)}}(x),
\]
where \(\omega\) includes initialization and all SGD samples.

The deterministic dimension complexity \(\operatorname{dc}(\mathcal H)\)
is the least dimension of one feature map that exactly represents every
\(h\in\mathcal H\) by a tie-resolved homogeneous halfspace.  Its confident
variant \(\operatorname{dc}^{1/2}(\mathcal H)\) is the least \(d\) for which
a feature-map law \(\mathcal P\), chosen before \((\mathcal D,h)\), satisfies
\[
\Pr_{\varphi\sim\mathcal P}\!\left[
\inf_{w\in\mathbb R^d}
\Pr_{x\sim\mathcal D}
\{\operatorname{sign}_{s_0}(\langle w,\varphi(x)\rangle)h(x)<0\}=0
\right]\geq\frac12
\]
for every \(\mathcal D\) and \(h\).  Let
\(\varphi_{\rm id}(x)=x\in\mathbb R^n\).

\paragraph{Technical assumptions.}
\begin{assumption}[\textbf{Antipodally odd target class}]
\label{assump:antipodal-oddness}
Every \(h\in\mathcal H\) satisfies \(h(-x)=-h(x)\) for all
\(x\in\mathcal X\).
\end{assumption}

\begin{assumption}[\textbf{Strict high-accuracy regime}]\label{assump:high-accuracy}
The accuracy obeys the strict inequality
\(2\varepsilon<1/(n+1)\).
\end{assumption}

\begin{assumption}[Universal source success]
\label{assump:universal-sgd-success}
The width, stepsize, and horizon are fixed before the distribution and target,
and for every \(\mathcal D\) and \(h\in\mathcal H\),
\[
\mathbb E_{\omega}\bigl[\mathcal L_{\mathcal D,h}(G_\omega)\bigr]
\leq\varepsilon,
\]
where expectation is joint over the independent Gaussian initialization and
fresh one-sample SGD draws.
\end{assumption}

\paragraph{Main theory.}
\begin{theorem}[Exact identity representation in the odd high-accuracy regime]
Under Assumptions~\ref{assump:antipodal-oddness},
\ref{assump:high-accuracy}, and~\ref{assump:universal-sgd-success}, the
identity map exactly represents the class:
\[
\forall h\in\mathcal H\ \exists w_h\in\mathbb R^n\ \forall x\in\mathcal X,
\qquad
\operatorname{sign}_{s_0}(\langle w_h,x\rangle)=h(x).
\]
Hence the distribution- and target-independent law
\(\mathcal P_{\rm id}=\delta_{\varphi_{\rm id}}\) succeeds with probability
one for every \((\mathcal D,h)\), and
\[
\operatorname{dc}^{1/2}(\mathcal H)
\leq\operatorname{dc}(\mathcal H)\leq n\leq S\leq TS.
\]
\end{theorem}

\paragraph{Discussion.}
This is \textbf{partial progress}.  It preserves the Boolean domain,
distribution-independent learner parameters, Gaussian initialization, and
all-layer SGD.  It differs from the source problem in three essential ways:
\begin{enumerate}[label=(\roman*)]
\item the network is bias-free and has depth two;
\item every target is antipodally odd;
\item the accuracy satisfies $2\varepsilon<1/(n+1)$.
\end{enumerate}

\noindent \textbf{Technical role and remaining barrier.}
Under the source setting, the proof faces three technical barriers.  First,
the antisymmetric part of a general deep ReLU score can remain nonlinear, so
it does not directly define a shared linear feature map.  Second, an arbitrary
target may assign incompatible labels to antipodal inputs, blocking the
antipodal reduction.  Third, the source condition $\varepsilon<1/4$ is too weak
to make a finite infeasibility witness contradict the learner guarantee.  The
added restrictions resolve these barriers one by one: the bias-free depth-two
architecture and the ReLU identity $\sigma(z)-\sigma(-z)=z$ linearize the
antisymmetric score; target oddness aligns the labels on antipodal pairs; and
$2\varepsilon<1/(n+1)$ supplies the strict finite-witness gap.  Together they
yield an exact representation by the identity features.  The remaining
problem is to obtain an $O(TS)$ common linear representation under the source
assumptions, without these three restrictions.

\paragraph{\underline{Perspective 1.}}

\paragraph{Formalized setting and preliminaries.}
Let \(n\) be a positive integer, set \(\mathcal X=\{-1,+1\}^n\), and let
\(\mathcal H\subseteq\{-1,+1\}^{\mathcal X}\).  Write
\(\mathbb N_0=\{0,1,2,\ldots\}\) and \([r]=\{1,\ldots,r\}\) for every
positive integer \(r\).  Fix the source tie label
\(s_0\in\{-1,+1\}\), and define
\[
\operatorname{sign}_{s_0}(z)=
\begin{cases}
+1,&z>0,\\
-1,&z<0,\\
s_0,&z=0.
\end{cases}
\]
Set
\[
R_{\mathcal D,h}(w,\varphi)
:=\Pr_{x\sim\mathcal D}
\{\operatorname{sign}_{s_0}(\langle w,\varphi(x)\rangle)h(x)<0\}.
\]
For \(\alpha\geq0\), define
\[
\operatorname{dc}_\alpha(\mathcal H)
:=\min\left\{d\in\mathbb N_0:\ \exists\mathcal P\ \forall\mathcal D\
\forall h\in\mathcal H,\quad
\mathbb E_{\varphi\sim\mathcal P}
\left[\inf_{w\in\mathbb R^d}R_{\mathcal D,h}(w,\varphi)\right]
\leq\alpha\right\},
\]
where the law \(\mathcal P\) is selected before \((\mathcal D,h)\).

Consider a bias-free, fully connected ReLU network of positive integer depth
\(L\).  Its widths \(n_0=n,n_1,\ldots,n_{L-1},n_L=1\) are positive integers,
and its parameter count is
\(S=\sum_{\ell=1}^Ln_\ell n_{\ell-1}\).  With
\(z_0=x\), \(u_\ell=\theta_\ell z_{\ell-1}\),
\(z_\ell=\max\{0,u_\ell\}\) for \(\ell<L\), and
\(f_\theta=\theta_Lz_{L-1}\), initialize independently with
\[
(\theta_\ell^{(0)})_{jk}\sim\mathcal N(0,1/n_{\ell-1}).
\]
Fix a
ReLU-kink selector \(\kappa\in[0,1]\): derivatives at positive, negative,
and zero preactivations are \(1\), \(0\), and \(\kappa\), respectively.  For fresh
\(x^{(t)}\sim\mathcal D\), run the exact all-layer recursion
\[
\theta^{(t+1)}=\theta^{(t)}-
\eta\nabla_\theta^{(\kappa)}
\ell\!\left(h(x^{(t)})f_{\theta^{(t)}}(x^{(t)})\right),
\qquad \ell(a)=\log(1+e^{-a}),
\]
and return the tie-resolved latter-half score
\(A(x)=\sum_{t=\lceil T/2\rceil}^{T}f_{\theta^{(t)}}(x)\).

For \(r\geq0\), let
\(B_\infty(\theta^{(0)},r)=\{\theta:
\|\theta-\theta^{(0)}\|_\infty\leq r\}\), and define
\[
M_r(\theta^{(0)})=
\begin{cases}
+\infty,&L=1,\\
\displaystyle\inf_{\substack{\theta\in B_\infty(\theta^{(0)},r),\
x\in\mathcal X,\\
1\leq\ell<L,\ j\in[n_\ell]}}
|u_{\ell,j}(\theta,x)|,&L\geq2,
\end{cases}
\]
\[
G_r(\theta^{(0)})=
\sup_{\substack{\theta\in B_\infty(\theta^{(0)},r),\
x\in\mathcal X,\\ y\in\{-1,+1\}}}
\|\nabla_\theta^{(\kappa)}\ell(yf_\theta(x))\|_\infty,
\qquad
E_r=\{M_r>0,\ \eta TG_r\leq r\}.
\]
\textbf{This event depends only on initialization and worst-case quantities over the
fixed ball}.  Finally, for paths
\(p=(i_0,\ldots,i_{L-1})\in\prod_{\ell=0}^{L-1}[n_\ell]\), define
\[
[\varphi_{\theta^{(0)}}(x)]_p
=x_{i_0}\prod_{\ell=1}^{L-1}
\mathbf1\{u_{\ell,i_\ell}(\theta^{(0)},x)>0\},
\qquad
d_{\rm path}=\prod_{\ell=0}^{L-1}n_\ell.
\]
Let \(\mathcal P_{\rm gate}\) be the unconditional law of this map under
Gaussian initialization, defined on both \(E_r\) and \(E_r^c\).

\paragraph{Technical assumptions.}
\begin{assumption}[Fixed source witnesses and regime]
\label{assump:fixed-source-witnesses}
\(0\leq\varepsilon<1/4\), \(T\) is a positive integer, and \(\eta>0\).  For the given
\((n,\mathcal H,\varepsilon)\), the architecture, \(S,\eta,T\), and the
fixed tie and kink conventions are selected once before every
\((\mathcal D,h)\).
\end{assumption}

\begin{assumption}[Universal expected-error SGD premise]
\label{assump:universal-expected-success}
For every distribution \(\mathcal D\) on \(\mathcal X\) and every
\(h\in\mathcal H\), the strict classification error of the prescribed
latter-half predictor satisfies
\[
\mathbb E_{\theta^{(0)},x^{(0)},\ldots,x^{(T-1)}}
\left[
\Pr_{x\sim\mathcal D}
\{\operatorname{sign}_{s_0}(A(x))h(x)<0\}
\right]
\leq\varepsilon.
\]
\end{assumption}

\begin{assumption}[\textbf{Fixed constant depth}]\label{assump:constant-depth}
A universal positive integer \(L_0\), independent of all problem and learner
parameters, satisfies \(1\leq L\leq L_0\).  Widths are arbitrary and all
layers remain trainable.
\end{assumption}

\begin{assumption}[\textbf{Static robust initialization tube}]\label{assump:robust-tube}
Deterministic \(r>0\) and \(0\leq\delta_0\leq\varepsilon\) are fixed before
initialization, distribution, target, and SGD samples, and
\[
\Pr_{\theta^{(0)}}(E_r)\geq1-\delta_0.
\]
This initialization-only event is defined by worst-case quantities over the
fixed ball.  Trajectory containment, gate stability, and the path
representation are derived from \(E_r\).
\end{assumption}

\paragraph{Main theory.}
\begin{theorem}[Conditional polynomial probabilistic dimension]
Under Assumptions~\ref{assump:fixed-source-witnesses},
\ref{assump:universal-expected-success},
\ref{assump:constant-depth}, and~\ref{assump:robust-tube}, the single
unconditional law \(\mathcal P_{\rm gate}\), chosen before every distribution
and target, satisfies for all \(\mathcal D\) and \(h\in\mathcal H\)
\[
\mathbb E_{\varphi\sim\mathcal P_{\rm gate}}
\left[\inf_{w\in\mathbb R^{d_{\rm path}}}
R_{\mathcal D,h}(w,\varphi)\right]
\leq\varepsilon+\delta_0.
\]
Consequently,
\[
\operatorname{dc}_{\varepsilon+\delta_0}(\mathcal H)
\leq d_{\rm path}\leq S^L\leq S^{L_0},
\qquad
\operatorname{dc}_{2\varepsilon}(\mathcal H)\leq S^{L_0}.
\]
The learner premise averages over initialization and the fixed finite SGD
horizon; the tube premise is an initialization probability; and the conclusion
averages over the unconditional feature-map law.
\end{theorem}

\paragraph{Discussion.}
This is \textbf{partial progress}.  It preserves one
Gaussian-initialized ReLU learner, finite-horizon all-layer SGD, and uniform
expected error over distributions and targets.  It differs from the source
problem in four ways:
\begin{enumerate}[label=(\roman*)]
\item the initialization must satisfy a robust-tube condition with high
probability;
\item the network depth is bounded by a constant $L_0$;
\item the network is bias-free;
\item the target is weakened to randomized approximate dimension
$\operatorname{dc}_{\varepsilon+\delta_0}(\mathcal H)\le S^{L_0}$.
\end{enumerate}

\noindent\textbf{Technical role and remaining barrier.}
The source setting leaves three technical barriers.  Instance-dependent SGD
can change hidden gates, so trajectories need not share a fixed feature map;
unrestricted depth can make the number of path features too large; and bad
initializations cannot be discarded when the target is deterministic and
exact.  The robust tube freezes the gates, the depth cap bounds the path-feature
dimension, and the bias-free architecture supports the displayed homogeneous
path-feature map; the randomized approximate target absorbs the bad-seed
event.  These devices give a polynomial feature bound, but not the source's
exact linear $O(TS)$ bound.  The remaining problem is to construct that common
representation without the tube, depth cap, bias-free restriction, or target
relaxation.

\subsubsection{Subproblem 2: Statistical Query Learning}
A $\tau-$statistical query (SQ) oracle for input distribution $\mathcal D$ and target $h^*$, on input query $q: \mathcal X \times \{\pm 1\} \to [-1, 1]$ and tolerance $\tau$ returns an arbitrary value $v$ such that $|v - \mathbb E_{x\sim \mathcal D} q(x, h^*(x))| \leq \tau.$ A (randomized) $(m, \tau)-$statistical query (SQ) algorithm operates by making a sequence of m queries to a $\tau-$SQ oracle where each query can depend on all previous responses and can be selected at random, and then returns a predictor $\hat h: \mathcal X \to \{\pm1\}.$
\begin{openquestion}
\textit{Is there a constant $C$ such that for every class $\mathcal H \subseteq \{\pm1\}^\mathcal X$, over any domain $\mathcal X$, and any $\epsilon < 1/4$, if there exists an $(m, \tau)-$SQ algorithm s.t. for every input distribution $\mathcal{D}$ and every $h^*\in \mathcal H$, the algorithm returns a predictor $\hat h$ with $\mathbb E \mathcal L_{\mathcal D,h^*} (\hat h) \leq \epsilon$ (expectation over the randomness of the algorithm), then $\operatorname{dc}(\mathcal{H}) \leq C \cdot m/\tau^2$.}
\end{openquestion}

\paragraph{\underline{Perspective 2.}}

\paragraph{Formalized setting and preliminaries.}
Let \((\mathcal X,\Sigma)\) be an arbitrary measurable space, and let
\(\mathcal H\) be a class of measurable maps from
\(\mathcal X\) to \(\{-1,+1\}\).  Write
\(\mathbb N_0=\{0,1,2,\ldots\}\).  For a probability measure \(\mathcal D\)
on \((\mathcal X,\Sigma)\), measurable target \(h\), and measurable binary
predictor \(g\), write
\(\mathcal L_{\mathcal D,h}(g)=\Pr_{x\sim\mathcal D}\{g(x)\ne h(x)\}\).
Dimension complexity is the least \(d\in\mathbb N_0\) for which one map
\(\varphi:\mathcal X\to\mathbb R^d\), measurable with respect to the Borel
sigma-algebra on \(\mathbb R^d\), satisfies
\[
\forall h\in\mathcal H\ \exists w_h\in\mathbb R^d\ \forall x\in\mathcal X,
\qquad h(x)\langle w_h,\varphi(x)\rangle>0.
\]
If no such finite \(d\) exists, set
\(\operatorname{dc}(\mathcal H)=+\infty\).

Fix a randomized learner \(A\) with hidden seed \(U\sim\mu_A\).  Conditional
on \(U=u\), it makes at most \(m\) adaptive unrestricted queries
\(q_t:\mathcal X\times\{-1,+1\}\to[-1,1]\), measurable with respect to the
product sigma-algebra, each selected from the seed and the preceding public
transcript, and then returns a measurable binary predictor.  All seed-,
transcript-, query-, and output-coordinate maps are assumed measurable.

A deterministic \emph{complete response rule} \(R\) assigns a value in
\([-1-\tau,1+\tau]\) to every public query-bearing history admitted by \(A\),
including histories not reached by a particular seed or reply sequence.  It
does not observe the seed except through the public transcript.  Let
\(\mathfrak R_A^{\rm all}\) contain all such rules, and let \(g_{u,R}\) be the
terminal predictor.
A complete rule is \((\mathcal D,h,\tau)\)-valid if, at every round reached
by every seed interacting with it,
\[
\left|v_t-\mathbb E_{x\sim\mathcal D}q_t(x,h(x))\right|\leq\tau.
\]
Denote the valid rules by \(\mathfrak R_{A,\tau}(\mathcal D,h)\).  Accuracy
below is required pointwise for every valid response rule.

For every complete rule, define the seed-averaged terminal response and the
\textbf{static all-rule space}
\[
F_R(x)=\mathbb E_{U\sim\mu_A}[g_{U,R}(x)],\qquad
V_A=\operatorname{span}_{\mathbb R}
\{F_R:R\in\mathfrak R_A^{\rm all}\},\qquad
r_A=\dim V_A.
\]
Thus \(V_A\) spans the seed-averaged terminal responses over all complete
response rules.

\paragraph{Technical assumptions.}
\begin{assumption}[Primitive parameter regime]\label{assump:parameter-regime}
\(m\in\mathbb N_0\), \(\tau>0\), and
\(\varepsilon\in[0,1/4)\).  Fixed numerical constants \(B\geq1\) and
\(k\geq1\) are independent of the domain, class, learner instance, parameters,
distributions, targets, response rules, and seeds.
\end{assumption}

\begin{assumption}[Fixed randomized adaptive unrestricted-SQ interface]
\label{assump:adaptive-sq-interface}
The learner \(A\) is fixed before \(\mathcal D\), \(h\), and the response
policy.  It uses at most \(m\) adaptive bounded unrestricted queries, may use
the full preceding real-valued public transcript, and interacts measurably
with every deterministic complete response rule.
\end{assumption}

\begin{assumption}[Every-valid-policy universal guarantee]
\label{assump:universal-adversarial-guarantee}
For every distribution \(\mathcal D\), every \(h\in\mathcal H\), and every
\(R\in\mathfrak R_{A,\tau}(\mathcal D,h)\),
\[
\mathbb E_{U\sim\mu_A}\mathcal L_{\mathcal D,h}(g_{U,R})
\leq\varepsilon.
\]
The expectation is over the learner seed, uniformly for every valid \(R\).
\end{assumption}

\begin{assumption}[\textbf{Static polynomial mean-response-rank certificate}]
\label{assump:mean-response-rank}
The rank over all deterministic complete response rules is finite and obeys
\[
r_A\leq B\bigl(1+m/\tau^2\bigr)^k.
\]
This primitive certificate on the seed-averaged responses is fixed before the
learning instance.
\end{assumption}

\paragraph{Main theory.}
\begin{theorem}[Conditional static mean-response-rank theorem]
Under Assumptions~\ref{assump:parameter-regime}--
\ref{assump:mean-response-rank}, choose once a basis
\(\psi_1,\ldots,\psi_{r_A}\) of \(V_A\) and define
\(\varphi_A(x)=(\psi_1(x),\ldots,\psi_{r_A}(x))\).  This deterministic map
is independent of the distribution, target, valid response policy, and
realized seed.  For every \(h\in\mathcal H\), some
\(w_h\in\mathbb R^{r_A}\) satisfies
\[
h(x)\langle w_h,\varphi_A(x)\rangle
\geq1-2\varepsilon>\frac12>0
\qquad\text{for every }x\in\mathcal X.
\]
Consequently,
\[
\operatorname{dc}(\mathcal H)
\leq r_A\leq B\bigl(1+m/\tau^2\bigr)^k.
\]
The conclusion is deterministic and pointwise.
\end{theorem}

\paragraph{Discussion.}
This is \textbf{partial progress}.  It preserves arbitrary domains,
adaptive randomized SQ learning, adversarial tolerance-valid replies, and
uniform success over distributions and targets.  \emph{The decisive added condition
is a static polynomial-rank certificate for the seed-averaged terminal
responses.}  The resulting dimension bound is polynomial,
$B(1+m/\tau^2)^k$, rather than the source target $O(m/\tau^2)$.

\noindent\textbf{Technical role and remaining barrier.}
The source SQ interface leaves a finite-dimensionality barrier: adaptive
real-valued transcripts do not by themselves place the terminal predictors in
a common finite-dimensional function space.  The static rank certificate
supplies that common space, so seed averaging and exactification can produce a
deterministic representation.  The remaining technical problem is to derive
the source-scale compression from the SQ interface itself: if
$F^0_{\mathcal D,h}$ is the seed-averaged terminal predictor under the
canonical exact-center policy, then
\[
 \dim\operatorname{span}\{F^0_{\mathcal D,h}:\mathcal D,h\}
 \le C m/\tau^2.
\]

\subsection{Does Differential Privacy Make PAC Learning Much Harder?}
\begin{openquestion}
Denote $VC(C)$ the Vapnik-Chervonenkis dimension for a hypothesis class $C$, $LD(C)$ \citep{littlestone1988learning} the Littlestone dimension which is a combinatorial parameter characterizing online learnability. A randomized algorithm $A: (X\times \{0, 1\})^n\to \mathcal W$ is $(\epsilon, \delta)-$Differentially Private (DP) \citep{dwork2016calibrating} if for every pair of neighboring datasets $S, S'$
and every event $E \subseteq  \mathcal W$ it holds that $\mathrm{Pr}[A(S)\in E] \leq e^\epsilon \mathrm{Pr}[A(S')\in E] + \delta.$ A private learner
is a PAC learner that guarantees DP w.r.t. its training data. 
It is well known that DP learning requires
more samples than non-private learning for some classes. \citet{nissim2026private} asks the central question: \textit{How
much more is needed? Is the answer close to VC(C), or could it be drastically larger?}
\end{openquestion}

\subsubsection{Subproblem 1: The Sample Complexity of Private Learning}
\begin{openquestion}
\textit{Identify a combinatorial measure of a class that determines the sample complexity of privately learning it, analogously to the characterization of non-private learning in terms of the VC dimension.}
\end{openquestion}

\paragraph{\underline{Perspective 1.}}

\paragraph{Formalized setting and preliminaries.}
Let \((X,\Sigma)\) be a measurable space, \(C\subseteq\{0,1\}^X\), and \(\log\)
be natural.  A Cartesian partition makes restriction bijective onto the
blockwise product; it is finest if it refines every such partition.  For
\(t\geq1\), let \(\log_2^{(0)}t=t\), let \(\log_2^{(r)}\) denote the
\(r\)-fold iterate of the base-two logarithm for \(r\geq1\), and write
\(\log^*t=\min\{r\geq0:\log_2^{(r)}t\leq1\}\).  For the \textbf{finite finest
Cartesian partition} specified below, write
\(X=\bigsqcup_{i=1}^kX_i\), \(C_i=\{c|_{X_i}:c\in C\}\), and, for a
candidate positive integer \(n\), set
\begin{equation}
\label{priv-eq-complexity}
d_i=\operatorname{LD}(C_i),\quad s_i=1+\log^*(d_i+1),\quad
M_\oplus(C)=\sum_i s_i,\quad \omega_i=\frac{s_i}{M_\oplus(C)},\quad
m_{n,i}=\max\{8,\lceil4n\omega_i\rceil\}.
\end{equation}
Identify \(x,x'\in X_i\) when every \(c_i\in C_i\) agrees on them, and
write \(Q_i=X_i/{\equiv_i}\), \(\kappa_i:X_i\to Q_i\), and
\(\Sigma_i=\{A\cap X_i:A\in\Sigma\}\), giving \(Q_i\) the discrete
sigma-field \(2^{Q_i}\).  Every \(c_i\in C_i\) has a unique
representative \(\bar c_i:Q_i\to\{0,1\}\) satisfying
\(c_i=\bar c_i\circ\kappa_i\); set
\[
\bar C_i=\{\bar c_i:c_i\in C_i\},\qquad
\mathcal H_i=\{0,1\}^{Q_i},\qquad
\mathcal H^\oplus=\prod_i\mathcal H_i,
\]
where \(\mathscr H_i\) is the product sigma-field generated by finite
evaluation cylinders and \(\mathscr H^\oplus=\bigotimes_i\mathscr H_i\).
A quotient tuple \(\bar h=(\bar h_i)_i\in\mathcal H^\oplus\) decodes as
\(h_{\bar h}(x)=\bar h_i(\kappa_i(x))\) for \(x\in X_i\).

For a probability measure \(D\) on \((X,\Sigma)\) and \(c\in C\), let
\(R_D(h,c)=\Pr_D[h(x)\ne c(x)]\) and let
\(D_c\) be the law of \((x,c(x))\).  An unrestricted learner is a Markov
kernel on fixed-size replacement-adjacent samples, including nonrealizable
ones, with arbitrary measurable
output \(\Omega\) and decoder \(h_\Omega\); it may be joint, improper, or
computationally unbounded.  Finite evaluations and stated risk events are
measurable.  Define \(\operatorname{SC}_{\varepsilon,\delta}(C)\) as the
least positive integer \(n\) for which such a private learner satisfies, for
every \(c\in C\) and probability measure \(D\) on \((X,\Sigma)\),
\begin{equation}
\label{priv-eq-sample-complexity}
\Pr_{S\sim D_c^n,\,\Omega\sim A_n(S,\cdot)}
  [R_D(h_\Omega,c)\leq1/16]\geq15/16.
\end{equation}

\paragraph{Technical assumptions.}
\begin{assumption}[\textbf{Canonical finite Cartesian factorization}]
\label{priv-assump-cartesian}
There is a finite finest Cartesian partition
\(X=\bigsqcup_{i=1}^kX_i\), \(k\geq1\), with
\(C_i=\{c|_{X_i}:c\in C\}\) and \(C\simeq\prod_{i=1}^k C_i\).
\end{assumption}
\begin{assumption}[\textbf{VC-one factors}]
\label{priv-assump-factors}
Every \(C_i\) is nonconstant, has \(\operatorname{VC}(C_i)=1\), and has
finite \(d_i\).
\end{assumption}
\begin{assumption}[\textbf{Measurable countable evaluation quotients}]
\label{priv-assump-quotients}
Each \(X_i\in\Sigma\), each \(Q_i\) is finite or countable, and each
cell \(\kappa_i^{-1}(\{q\})\) belongs to \(\Sigma_i\).  Equivalently,
\(\kappa_i:(X_i,\Sigma_i)\to(Q_i,2^{Q_i})\) is measurable.
\end{assumption}
Consequently, every target and decoded hypothesis is \(\Sigma\)-measurable,
and quotient decoding preserves distributional zero-one risk exactly.
\begin{assumption}[Approximate-DP range]
\label{priv-assump-privacy}
\(0<\varepsilon\leq1/10\) and \(0<\delta<1\).
\end{assumption}
\begin{assumption}[Candidate-wise lower-bound budget]
\label{priv-assump-candidate}
At the lower-bound candidate \(n\),
\begin{equation}
\label{priv-eq-candidate-delta}
0<\delta\leq\min\left\{\frac1{n\log(n+1)},
\min_i\frac{c_\delta}{m_{n,i}^2\log(m_{n,i}+1)}\right\},
\end{equation}
where \(c_\delta>0\) is universal.
\end{assumption}

\paragraph{Main theory.}
\begin{theorem}[Conditional private direct sum for Cartesian VC-one products]
\label{priv-thm-direct-sum}
Under Assumptions~\ref{priv-assump-cartesian}--\ref{priv-assump-privacy},
let \(K_Y\) be the universal constant of the quotient-first VC-one learner
of \citet{yan2025vc1}, and define
\[
q_i=\left\lceil\frac{K_Ys_i}{\varepsilon}
 \log^2\!\left(\frac{es_i}{\varepsilon\delta}\right)\right\rceil,
\qquad Q_\oplus=\sum_iq_i.
\]
With \(C_{\rm up}=65536\) and
\(C_{\rm quota}=\max\{1,K_Y+1/20\}\), route the first \(q_i\) block
records with fixed padding, run fixed permutation-symmetrized, fully
totalized factor learners for \(\bar C_i\) into
\((\mathcal H_i,\mathscr H_i)\) at
\((\varepsilon/2,\delta/2)\), form their tuple in
\((\mathcal H^\oplus,\mathscr H^\oplus)\), and decode.
This is a measurable all-input replacement-\((\varepsilon,\delta)\)-DP
learner.  For \(n\geq\lceil C_{\rm up}Q_\oplus\rceil\),
\eqref{priv-eq-sample-complexity} holds and
\begin{equation}
\label{priv-eq-upper}
\operatorname{SC}_{\varepsilon,\delta}(C)
\leq\lceil C_{\rm up}Q_\oplus\rceil
\leq C_{\rm up}C_{\rm quota}\frac{M_\oplus(C)}{\varepsilon}
\log^2\!\left(\frac{eM_\oplus(C)}{\varepsilon\delta}\right).
\end{equation}
The upper bound holds throughout \(0<\delta<1\).

Conversely, for universal \(c_{\rm low}>0\), every candidate satisfying
\eqref{priv-eq-candidate-delta} and every unrestricted measurable
replacement-\((\varepsilon,\delta)\)-DP learner satisfying
\eqref{priv-eq-sample-complexity} obey
\(n\geq c_{\rm low}M_\oplus(C)\).  Below this threshold some full-product
target and arbitrary-support \(D\) with \(D(X_i)=\omega_i\) satisfy
\begin{equation}
\label{priv-eq-lower-witness}
\Pr_{S\sim D_c^n,\,\Omega\sim A_n(S,\cdot)}
 [R_D(h_\Omega,c)>1/16]>1/16.
\end{equation}
When \eqref{priv-eq-candidate-delta} holds at
\(n_*=\operatorname{SC}_{\varepsilon,\delta}(C)\), the lower bound applies at
the sample-complexity threshold and combines with \eqref{priv-eq-upper}.
\end{theorem}

\paragraph{Discussion.}
This is \textbf{partial progress}.  It preserves realizable binary
PAC learning, arbitrary targets and data distributions, approximate privacy,
and unrestricted learners in the lower bound.  \emph{The main restriction is to
Cartesian products of VC-one, finite-Littlestone factors with countable
evaluation quotients.}  Within this subclass it gives matching bounds up to
privacy and logarithmic factors; the lower bound additionally requires the
stated small-$\delta$ condition.

\noindent \textbf{Technical role and remaining barrier.}
The source setting leaves two structural barriers: a general finite-
Littlestone class has no factorwise decomposition to which threshold bounds
apply, and arbitrary evaluation spaces can make the learner construction
non-measurable.  Cartesian VC-one structure supplies the factorwise direct
sum, while countable evaluation quotients provide a measurable
quotient-first learner.  The small-$\delta$ condition separately makes the
threshold lower bound applicable.  The remaining problem is to obtain a
comparable intrinsic characterization for every measurable class of finite
Littlestone dimension, without the Cartesian, VC-one, or countable-quotient
assumptions and without the candidate-wise small-$\delta$ restriction.
\subsubsection{Subproblem 2: Class Existence}
\begin{openquestion}
\textit{Does there exist a sequence of finite-size classes $C = \{C_\kappa\}_{\kappa\in \mathbb N}$ where:
(1) $\lim_{\kappa\to \infty} |C_\kappa| = \infty,$ (2) $\log |C_\kappa|$ is superpolynomial in $VC(C_\kappa)$, and (3) The number of
samples required to learn $C_\kappa$ under differential privacy is $\Omega(\log |C_\kappa|)$.}
\end{openquestion}

\paragraph{\underline{Perspective 2.}}

\paragraph{Formalized setting and preliminaries.}
Let \((X,\Sigma)\) be a measurable space, and let \(C\) be a finite class of
\(\Sigma\)-measurable maps from \(X\) to \(\{0,1\}\).  For \(t\in[N]\), let
\(\tau_t(q)=\mathbf 1\{q\leq t\}\).  The class
\(C\) has a \((k,N)\) \textbf{disjoint-threshold minor}
if there are injections \(\phi_j:[N]\to X\) with disjoint images and
concepts \(c_{\mathbf t}\in C\), \(\mathbf t\in[N]^k\), such that
\[
c_{\mathbf t}(\phi_j(q))=\tau_{t_j}(q)
\quad(j\in[k],\ q\in[N]).
\]
For a probability law \(Q\) on the product measurable space
\(X\times\{0,1\}\), write
\(L_Q(h)=\Pr_{(x,y)\sim Q}[h(x)\ne y]\) for measurable \(h\).  Learners use
exactly an integer \(m\geq1\) iid examples and replacement adjacency; their
output hypotheses and all output-coordinate evaluation maps are measurable.
For a probability measure \(P\) on \((X,\Sigma)\) and \(c\in C\), let
\(Q_{P,c}\) be the law of \((x,c(x))\) for \(x\sim P\), and define
\(\log^*t=\min\{r\geq0:\log_2^{(r)}t\leq1\}\).
Here \(\log\) is natural.

\paragraph{Technical assumptions.}
\begin{assumption}[\textbf{Disjoint-threshold restriction table}]
\label{assump:minor-table}
Integers \(k\geq1\), \(N\geq2\), and \(C\) admit the minor defined above.
\end{assumption}
\begin{assumption}[Unrestricted private PAC antecedent]
\label{assump:unrestricted-private-pac}
The kernel \(A:(X\times\{0,1\})^m\rightsquigarrow\{0,1\}^X\) is
replacement-\((\varepsilon_0,\delta_m)\)-DP and, for every \(c\in C\)
and distribution \(P\) on \(X\),
\[
\Pr_{S\sim Q_{P,c}^m,\,h\sim A(S)}
[L_{Q_{P,c}}(h)>\alpha_0]\leq\beta_0.
\]
The learner may be improper and computationally unbounded.
\end{assumption}
\begin{assumption}[Fixed PAC and approximate-privacy scale]
\label{assump:fixed-parameter-scale}
\[
\alpha_0=\frac1{128},\qquad \beta_0=\frac1{32},\qquad
0<\varepsilon_0\leq1,\qquad
0\leq\delta_m\leq\frac{c_\delta}{m^2\log(em)},
\]
where the theorem fixes the universal constant \(c_\delta\).
\end{assumption}

\paragraph{Main theory.}
\begin{theorem}[Private direct-sum threshold-minor lower bound]
There are universal constants \(c,c_\delta>0\) and \(N_0\geq2\) such
that, for every \(k\geq1\), \(N\geq N_0\), domain \(X\), finite class
\(C\), and fixed-sample learner satisfying
Assumptions~\ref{assump:minor-table}--\ref{assump:fixed-parameter-scale},
\[
m\geq c\,k\log^*N,
\]
where \(\log^*\) is the base-two iterated logarithm.
\end{theorem}

\paragraph{Discussion.}
This is \textbf{partial progress}.  It preserves finite binary
classes, distribution-free realizable PAC learning, approximate privacy, and
a lower bound against unrestricted learners.  It differs from the source
problem in two essential ways:
\begin{enumerate}[label=(\roman*)]
\item the class must contain a $(k,N)$ disjoint-threshold minor, under a
small-$\delta_m$ privacy regime;
\item the conclusion is $\Omega(k\log^*N)$ rather than the requested
$\Omega(\log|C|)$ existential separation.
\end{enumerate}

\noindent\textbf{Technical role and remaining barrier.}
The source class-existence question leaves two barriers: an arbitrary finite
class need not contain independent threshold-hard directions, and combining
many hard directions can violate record-level privacy through record reuse or
group privacy.  The disjoint-threshold minor exposes independent hard
coordinates, and the one-use hidden-arm reduction combines them while
preserving record-level adjacency.  The small-$\delta_m$ regime makes the
component threshold lower bound applicable.  The remaining problem is to
construct finite classes with
$\log|C_\kappa|$ superpolynomial in $\operatorname{VC}(C_\kappa)$ for which
every private learner requires $\Omega(\log|C_\kappa|)$ samples without relying
on the added small-$\delta_m$ regime.

\subsection[Online Optimization of Piecewise-Lipschitz Functions with Applications to Data-Driven Algorithm Design]
{Online Optimization of Piecewise-Lipschitz Functions with Applications to\\Data-Driven Algorithm Design}
\subsubsection{Subproblem 1: Polynomial Boundaries}
\begin{openquestion}
For a monic polynomial
\[
\phi_\alpha(\theta)
=\theta^d+\alpha_{d-1}\theta^{d-1}+\cdots+\alpha_0,
\qquad \alpha\in[-R,R]^d,\qquad \theta \in \Theta,
\]
for a coefficient distribution
class $\mathcal{D}$,
the relevant constant for a coefficient-law class $\mathcal D$ is
\[
C_{\mathcal D}
=\sup_{\mu\in\mathcal D}
 \sup_{\substack{I\subseteq\Theta\text{ interval}\\|I|>0}}
 \frac{\Pr_{\alpha\sim\mu}
 [\exists\theta\in I:\phi_\alpha(\theta)=0]}{|I|}.
\]
The first subproblem of \citet{balcan2026online} asks: \textit{under what natural necessary
and sufficient conditions on $\mathcal D$ is $C_{\mathcal D}$ finite, and under what conditions is it bounded by a polynomial in $d$ and $R$?}
\end{openquestion}

\paragraph{\underline{Perspective 2.}}

\paragraph{Formalized setting and preliminaries.}
Let \(\Theta\subseteq\R\) be compact and
\(\mathcal I(\Theta)\) its positive finite-length intervals, with arbitrary
endpoint conventions.  For
\(d\geq1\), \(R\geq1\), and
\(\alpha=(\alpha_0,\ldots,\alpha_{d-1})\in[-R,R]^d\), set
\[
\phi_\alpha(\theta)=\theta^d+\sum_{j=0}^{d-1}\alpha_j\theta^j,
\qquad H_{d,I}=\{\alpha:\exists\theta\in I,\ \phi_\alpha(\theta)=0\}.
\]
For a law \(\mu\), define the random regular conditional-density norms
\[
K_0^\mu=\|f^\mu_{\alpha_0\mid\alpha_{1:d-1}}
(\,\cdot\mid\alpha_{1:d-1})\|_{L^\infty(\R)},
\qquad
K_\infty^\mu=\|f^\mu_{\alpha_{d-1}\mid\alpha_{0:d-2}}
(\,\cdot\mid\alpha_{0:d-2})\|_{L^\infty(\R)},
\]
evaluated at the random conditioning coordinates (for \(d=1\), both are
the density norm of \(\alpha_0\)).  For fixed finite
\(\eta=(\bar\kappa_0,\bar\kappa_\infty)\), let
\(\mathcal D_{d,R,\eta}\) contain exactly the laws supported on
\([-R,R]^d\) whose two regular conditional densities exist almost surely and obey
\(\mathbb E K_0^\mu\leq\bar\kappa_0\) and
\(\mathbb E K_\infty^\mu\leq\bar\kappa_\infty\).  These are mean caps,
rather than almost-sure uniform caps.  Thus \(\mathcal D_{d,R,\eta}\) permits
dependent laws specified through endpoint regular conditional densities.
Define
\[
C_{\mathcal D_{d,R,\eta}}=
\sup_{\mu\in\mathcal D_{d,R,\eta}}\sup_{I\in\mathcal I(\Theta)}
\frac{\mu(H_{d,I})}{|I|}.
\]
All interval lengths are Lebesgue lengths, and the definition assigns value
zero whenever either indexing family is empty.
For \(I_0=I\cap[-1,1]\), \(I_+=I\cap(1,\infty)\), and
\(I_-=I\cap(-\infty,-1)\), set
\[
B_0=d+\frac{Rd(d-1)}2,\quad B_\infty=1+\frac{Rd(d-1)}2,
\quad M_\eta=\max\{\bar\kappa_0B_0,\bar\kappa_\infty B_\infty\},
\quad \bar\kappa_*=\max\{\bar\kappa_0,\bar\kappa_\infty\}.
\]
Define a witness law by taking \(\alpha_0\) uniform on \([-R,R]\) when
\(d=1\), take \(\alpha_0,\alpha_1\) independently uniform when \(d=2\),
and, for \(d\geq3\), take independent uniform endpoint coefficients and
set every middle coefficient to \(RS\) for one independent Rademacher \(S\).
Call the resulting law \(\mu^{\rm wit}_{d,R}\).

\paragraph{Technical assumptions.}
\begin{assumption}[Compact parameter domain]
\label{assump:compact-parameter-domain}
\(\Theta\subseteq\R\) is fixed and compact; the result is uniform in
\(I\in\mathcal I(\Theta)\) and independent of \(\Theta\).
\end{assumption}
\begin{assumption}[Indexed regime]
\label{assump:indexed-regime}
\(d\geq1\), \(R\geq1\), and finite \(\eta\in[0,\infty)^2\) is fixed
independently of \((d,R)\).
\end{assumption}
\begin{assumption}[Compact cube support and monicity]
\label{assump:compact-cube-support}
Every quantified law is supported on \([-R,R]^d\), and
\(\phi_\alpha\) is exactly monic.
\end{assumption}
\begin{assumption}[Mean endpoint conditional-density caps]
\label{assump:mean-endpoint-conditional-caps}
The two conditional densities and mean bounds defining
\(\mathcal D_{d,R,\eta}\) hold.
\end{assumption}

\paragraph{Main theory.}
\begin{theorem}[Endpoint conditional anti-concentration]
Under Assumptions~\ref{assump:compact-parameter-domain}--
\ref{assump:mean-endpoint-conditional-caps}, for every
\(\mu\in\mathcal D_{d,R,\eta}\) and \(I\in\mathcal I(\Theta)\),
\[
\mu(H_{d,I})\leq
\bar\kappa_0B_0|I_0|+
\bar\kappa_\infty B_\infty(|I_+|+|I_-|)
\leq M_\eta|I|.
\]
Consequently,
\[
C_{\mathcal D_{d,R,\eta}}
\leq M_\eta
\leq \bar\kappa_*d+\frac{\bar\kappa_*}{2}Rd^2.
\]
If
\(\bar\kappa_0,\bar\kappa_\infty\geq1/2\), then for every \(d,R\)
the witness law belongs to \(\mathcal D_{d,R,\eta}\) and
\(K_0^{\mu^{\rm wit}}=K_\infty^{\mu^{\rm wit}}=1/(2R)\) almost surely.
\end{theorem}

\paragraph{Discussion.}
This is \textbf{partial progress}.  It preserves monic degree-$d$
polynomials, cube-supported coefficients, and the source root-hitting
constant.  It differs from the source problem in three ways:
\begin{enumerate}[label=(\roman*)]
\item the coefficient law must satisfy endpoint conditional-density bounds;
\item the theorem is restricted to $R\ge1$;
\item it gives a sufficient polynomial upper bound, not a necessary-and-
sufficient characterization.
\end{enumerate}

\noindent\textbf{Technical role and remaining barrier.}
The source coefficient-law problem leaves two analytic barriers. General laws provide no controlled one-dimensional sweep near zero or infinity, and a direct joint-density argument can incur exponential dependence on the ambient coefficient dimension. The endpoint conditional-density bounds control the two endpoint charts separately, while the outer pivot removes the large-root growth, yielding the bound $O_\eta(Rd^2)$. The restriction $R\ge1$ has a separate role: when $\bar\kappa_0,\bar\kappa_\infty\ge1/2$, it makes the indexed family uniformly nonempty under the fixed caps, since the witness satisfies $K_0=K_\infty=1/(2R)\le1/2$. The local sweep bounds themselves extend to $R>0$ whenever the capped family is nonempty. The remaining problem is to obtain a natural coefficient-side condition that is necessary and sufficient, up to polynomial factors, for all $R>0$ and general coefficient-law classes, including the cap scaling needed in the small-$R$ regime.

% \newpage
\subsubsection{Subproblem 2: Pfaffian Boundaries}
\begin{openquestion}
Let $F=(F_1,\ldots,F_N)$ be a vector of Pfaffian functions on a compact
interval $\Theta$, and let
\[
\phi_\alpha(\theta)=\langle\alpha,F(\theta)\rangle,
\qquad \alpha\in[-R,R]^N.
\]
Suppose every law in $\mathcal D$ has a joint density bounded by $\kappa$.
The second subproblem of \citet{balcan2026online} asks: \textit{what normalization of $F$, analogous to
fixing the leading coefficient of a polynomial to one, guarantees that
$C^{\mathrm{Pf}}_{\mathcal D}$ is finite?}
Here
\[
C^{\mathrm{Pf}}_{\mathcal D}
=\sup_{\mu\in\mathcal D}
 \sup_{\substack{I\subseteq\Theta\text{ interval}\\|I|>0}}
 \frac{\Pr_{\alpha\sim\mu}
 [\exists\theta\in I:\phi_\alpha(\theta)=0]}{|I|}.
\]
\end{openquestion}

\paragraph{\underline{Perspective 1.}}

\paragraph{Formalized setting and preliminaries.}
Let \(\Theta=[c-h,c+h]\), \(x=(\theta-c)/h\in[-1,1]\), and use the
one-parameter convention \(p=1\).  Let
\(\eta=(\eta_1,\ldots,\eta_q)\) be a triangular Pfaffian chain with
polynomials \(P_j(x,y_{1:j})\).  For output polynomials
\(Q_i(x,y_{1:q})\), define
\[
G_i(x)=Q_i(x,\eta(x)),\qquad F_i(\theta)=G_i(x(\theta)),
\]
\[
M=\max_j\deg P_j,\quad \Delta=\max_i\deg Q_i,\quad
B_P=\max_j\|\operatorname{coeff}(P_j)\|_1,\quad
B_Q=\max_i\|\operatorname{coeff}(Q_i)\|_1,
\tag{F1}
\]
with total degrees and \(M=B_P=0\) if \(q=0\).  Let
\(\mathcal D_{N,R,\kappa}\) be all
laws on \([-R,R]^N\) with full joint density at most \(\kappa\); their
coordinates may be arbitrarily correlated.  Set
\[
\mathcal D=\mathcal D_{N,R,\kappa},\qquad
A=(2R)^N\kappa,\qquad \gamma_F=F/\|F\|_2,\qquad
\Gamma_{\rm proj}(F)=\operatorname*{ess\,sup}_{\theta\in\Theta}
\|\gamma_F'(\theta)\|_2,
\]
and let \(C^{\rm Pf}_{\mathcal D}(F;\Theta)\) be the supremum, first over
positive-length intervals \(I\subseteq\Theta\) and then over
\(\mu\in\mathcal D_{N,R,\kappa}\), of
\(\Pr_\mu[\exists\theta\in I:\langle\alpha,F(\theta)\rangle=0]/|I|\).

For an affine offset \(F_0\in C^1(\Theta)\) and a measurable partition
\(I=\bigsqcup_jE_j\) with \(F_j\ne0\) on \(E_j\), index
\(\beta\in[-R,R]^{N-1}\) by \(i\ne j\) and set
\[
T_j(\theta,\beta)=-\frac{F_0(\theta)}{F_j(\theta)}
-\sum_{i\ne j}\beta_i\frac{F_i(\theta)}{F_j(\theta)}.
\tag{F2}
\]
For \(N=1\), the zero-dimensional cube has Lebesgue mass one.

\paragraph{Technical assumptions.}
\begin{assumption}[Primitive parameter regime]
\label{assump:parameter-regime-50}
\(N\geq1\), \(q\geq0\), \(h,R>0\), \(0<\kappa<\infty\), the class
\(\mathcal D_{N,R,\kappa}\) is nonempty, and all degrees and coefficient
budgets in (F1) are finite.
\end{assumption}
\begin{assumption}[Balcan common-chain presentation]
\label{assump:balcan-common-chain}
Each \(\eta_j\in C^1([-1,1])\) satisfies
\(\eta_j'=P_j(x,\eta_{1:j})\), and every \(G_i=Q_i(x,\eta)\), with the
degree convention (F1).
\end{assumption}
\begin{assumption}[Literal anchor and unit-range certificate]
\label{assump:anchored-unit-range}
\(|\eta_j(x)|\leq1\) on \([-1,1]\) and \(Q_1\equiv1\).  Thus
\(F_1=G_1=1\), supplying the norm margin used below.
\end{assumption}
\begin{assumption}[Arbitrarily correlated capped joint laws]
\label{assump:cube-density-laws}
Every probabilistic clause quantifies over arbitrary
\(\mu\in\mathcal D_{N,R,\kappa}\), including dependent coordinate laws.
\end{assumption}
\begin{assumption}[Deterministic affine offset and pivot cover]
\label{assump:affine-chart-data}
The affine clause uses precisely the deterministic \(F_0\) and measurable
pivot partition described before (F2).
\end{assumption}

\paragraph{Main theory.}
\begin{theorem}[Anchored coefficient-normalized Pfaffian sweep]
Under Assumptions~\ref{assump:parameter-regime-50}--
\ref{assump:anchored-unit-range}, the following five conclusions hold; the
probabilistic ones also use Assumption~\ref{assump:cube-density-laws}, and
the affine one uses Assumption~\ref{assump:affine-chart-data}.

\emph{(R1) Projective speed.}  With
\(D_*=\Delta B_Q(1+qB_P)\), pointwise
\[
|G_i'|\leq D_*,\qquad \|G'\|_2\leq\sqrt N D_*,\qquad
\Gamma_{\rm proj}(F)\leq\frac{\sqrt N\Delta B_Q(1+qB_P)}h.
\tag{R1}
\]
The literal anchor gives \(\|F\|_2,\|G\|_2\geq1\), so this conditioning is
derived from the normalized presentation.

\emph{(R2) Central sweep.}  For every possibly correlated
\(\mu\in\mathcal D_{N,R,\kappa}\) and then
every positive-length \(I\subseteq\Theta\),
\[
\Pr_\mu[\exists\theta\in I:\langle\alpha,F(\theta)\rangle=0]
\leq A\sqrt{\frac N2}\Gamma_{\rm proj}(F)|I|
\leq\frac{AN\Delta B_Q(1+qB_P)}{\sqrt2h}|I|,
\tag{R2a}
\]
and the same two coefficients bound
\[
C^{\rm Pf}_{\mathcal D}(F;\Theta)
\leq A\sqrt{\frac N2}\Gamma_{\rm proj}(F)
\leq\frac{AN\Delta B_Q(1+qB_P)}{\sqrt2h}.
\tag{R2b}
\]

\emph{(R3) Affine sweep.}  For every
\(\mu\in\mathcal D_{N,R,\kappa}\), interval, and admissible pivot cover,
\[
\Pr_\mu[\exists\theta\in I:F_0(\theta)+\langle\alpha,F(\theta)\rangle=0]
\leq\kappa\sum_{j=1}^N\int_{E_j}\int_{[-R,R]^{N-1}}
|\partial_\theta T_j(\theta,\beta)|\,d\beta\,d\theta.
\tag{R3}
\]
The right side is interpreted in \([0,+\infty]\) and may equal \(+\infty\).

\emph{(R4) Exact monic recovery.}  For every integer \(d\geq1\) and bounded
interval \(J\subset\R\), choose \(\Theta\supseteq J\) and set
\(F_0(\theta)=\theta^d\), \(F_{k+1}(\theta)=\theta^k\)
for \(0\leq k<d\).  Thus
\(p_\alpha(\theta)=\theta^d+\sum_{k<d}\alpha_k\theta^k\), with the monic
coefficient deterministic and outside the random vector, and
\[
\begin{aligned}
Q_0(x)&=(c+hx)^d,
&Q_{k+1}(x)&=(c+hx)^k,
&\Delta_{\rm aug}&=d,\\
q&=M=B_P=0,
&N&=d,
&A&=(2R)^d\kappa.
\end{aligned}
\]
Using
\(E_1=J\cap\{|\theta|\leq1\}\) and
\(E_d=J\cap\{|\theta|>1\}\) (or \(E_1=J\) for \(d=1\)) in (R3) gives,
for every possibly correlated \(\mu\in\mathcal D_{d,R,\kappa}\),
\[
\Pr_\mu[\exists\theta\in J:p_\alpha(\theta)=0]
\leq\kappa(2R)^{d-1}\left(d+\frac{Rd(d-1)}2\right)|J|.
\tag{R4}
\]

\emph{(R5) Counterexample scale.}  For \(0<\delta\leq1\), take
\(F(\theta)=(1,\theta/\delta)\) on \([-1,1]\) and the uniform law on
\([-1,1]^2\), so \(\kappa=1/4\) and \(A=1\).  Then
\(B_Q=\Gamma_{\rm proj}(F)=1/\delta\) and, for every
\(0<\epsilon\leq\delta\),
\[
\Pr[\exists\theta\in[0,\epsilon]:
\alpha_1+\alpha_2\theta/\delta=0]=\frac{\epsilon}{4\delta},
\qquad
\frac1{4\delta}\leq
C^{\rm Pf}_{\mathcal D_{2,1,1/4}}(F;[-1,1])
\leq\frac1\delta\leq\frac{\sqrt2}{\delta}.
\tag{R5}
\]
When \(B_P\) is fixed, the dependence on \(M\) is degree zero.
\end{theorem}

\paragraph{Discussion.}
This is \textbf{full progress} for the declared anchored unit-range
normalization: the literal anchor and bounded-chain presentation yield
finite, explicit projective and central sweep bounds, a general affine chart
inequality, exact monic recovery, and the counterexample's unavoidable
\(1/\delta\) scale.  Whether every raw Pfaffian presentation admits this
normalization with polynomial parameter budgets remains open.

\section{Conclusion and Discussions}

\VALG organizes informal ML-theory research around source-relative theorem
branches whose assumptions, proof dependencies, and outcomes remain tied to
the originating question. Its global-to-local proof process supports both
local repair and formulation-level revision, allowing a failed proof mechanism
to produce a clearly scoped variant or relaxation. Across nine COLT 2026 case
studies, two runs produce internally finalized theorem candidates that match
the scope of the original subproblems, while the remaining runs yield
restricted-method results, special cases, or conditional theorems.

An important next step is to make AI-generated mathematics more readable and verifiable for researchers. Current systems tend to overuse notation and present proofs in ways that obscure the main argument, often deviating from established human proof-writing conventions. Even when derivations are plausible, these stylistic issues significantly increase the cost of expert verification. Rigorous evaluation also requires machine learning theory benchmarks that go beyond case studies of open problems. Our findings on COLT open problems indicate that difficult questions often benefit from multiple perspectives and evolving generation of ideas. However, to meaningfully assess the contribution and redundancy of these components, we need controlled benchmarks with known solutions, alternative valid formulations, explicit assumptions and relaxations, and clearly documented proof dependencies. Finally, formalization also remains a critical challenge. ML theory draws on probability, optimization, statistics, information theory, and learning theory, i.e., domains for which existing general-purpose formalization pipelines are often incomplete or cumbersome. Developing automated formalization tools tailored to the specific statements and proof patterns of ML theory would enable natural-language agents to explore formulations and proof strategies, while formal tools verify the mathematical core of their results. All of these further explorations and developments require the participation and effort of the entire ML theory community.

% \input{6_appendix}

% \newpage
\bibliographystyle{ims}
\bibliography{reference}
\end{document}